\documentclass[a4wide]{article}

\usepackage{textcomp}
\usepackage[utf8]{inputenc}
\usepackage[T1]{fontenc}

\usepackage[round]{natbib}

\usepackage{hyperref}
\usepackage[dvipsnames]{xcolor}
\hypersetup{
    colorlinks,
    linkcolor={red!60!black},
    citecolor={green!40!black}
}

\usepackage{amssymb}
\usepackage{amsmath}
\usepackage{amsthm}
\usepackage{mathtools}
\usepackage{algorithm,algorithmic}
\usepackage{stmaryrd}
\usepackage{booktabs}
\usepackage{color}
\usepackage{nicefrac}
\usepackage{multirow}

\DeclareMathOperator*{\argmin}{\mathrm{arg\,min}}
\DeclareMathOperator*{\argmax}{\mathrm{arg\,max}}

\newtheorem{theorem}{Theorem}[section]
\newtheorem{lemma}[theorem]{Lemma}

\begin{document}

\begin{center}
 {\Large Deep Neural Network Training with Frank--Wolfe}
\end{center}

\vspace{7mm}

\noindent\textbf{Sebastian Pokutta}\hfill\href{mailto:pokutta@zib.de}{\ttfamily pokutta@zib.de}\\
\emph{\small AI in Society, Science, and Technology \& Institute of Mathematics\\
Zuse Institute Berlin \& Technische Universit\"at Berlin \\
Berlin, Germany}\\
\\
\textbf{Christoph Spiegel}\hfill\href{mailto:spiegel@zib.de}{\ttfamily spiegel@zib.de}\\
\emph{\small AI in Society, Science, and Technology\\
Zuse Institute Berlin\\
Berlin, Germany}\\
\\
\textbf{Max Zimmer}\hfill\href{mailto:zimmer@zib.de}{\ttfamily zimmer@zib.de}\\
\emph{\small AI in Society, Science, and Technology\\
Zuse Institute Berlin\\
Berlin, Germany}\\
\\

\vspace{5mm}

\begin{center}
\begin{minipage}{0.85\textwidth}
\begin{center}
 \textbf{Abstract}
\end{center}
 {\small This paper studies the empirical efficacy and benefits of using projection-free first-order methods in the form of Conditional Gradients, a.k.a. Frank--Wolfe methods, for training Neural Networks with constrained parameters. We draw comparisons both to current state-of-the-art stochastic Gradient Descent methods as well as across different variants of stochastic Conditional Gradients.
	 
	In particular, we show the general feasibility of training Neural Networks whose parameters  are constrained by a convex feasible region using Frank--Wolfe algorithms and compare different stochastic variants. We then show that, by choosing an appropriate region, one can achieve performance exceeding that of unconstrained stochastic Gradient Descent and matching state-of-the-art results relying on $L^2$-regularization. Lastly, we also demonstrate that, besides impacting performance, the particular choice of constraints can have a drastic impact on the learned representations.}
\end{minipage}
\end{center}

\vspace{0mm}

\section{Introduction} \label{sec:introduction}

Despite its simplicity, stochastic Gradient Descent (SGD) is still the method of choice for training Neural Networks. A common assumption here is that the parameter space in which the weights $\theta$ of these networks lie is unconstrained. The standard SGD update can therefore simply be stated as
\begin{equation}
	\theta_{t+1} = \theta_t - \alpha \tilde{\nabla} L(\theta_t),
\end{equation}
where $L$ is some loss function to be minimized, $\tilde{\nabla} L(\theta_t)$ is the $t$-th batch or stochastic gradient, and $\alpha \geq 0$ the learning rate. In search of improved methods, the most notable modifications of this principle that have been proposed consist of adding momentum, see for example \cite{Qian1999} and \cite{Nesterov1983}, or of automatically adapting the learning rate on a per-parameter level, as is for example done by \cite{DuchiHazanSinger2011} and \cite{KingmaBa2014}, see also \cite{SchmidtSchneiderHennig2020} for a recent large-scale comparison. It has however been suggested that adaptive methods, despite initial speed ups, do not generalize as well as standard stochastic Gradient Descent in a wide variety of deep learning task, see \cite{WilsonRoelofsSternSrebroRecht2017}. In fact, for obtaining state-of-the-art test set performance on image classification datasets such as CIFAR-10 and ImageNet, a more significant contribution comes in the form of weight decay, see for example \cite{HansonPratt1989}. This regularization technique consists of modifying the weight update as
\begin{equation}
	\theta_{t+1} = (1-\lambda) \, \theta_t - \alpha \tilde{\nabla} L(\theta_t),
\end{equation}
where $\lambda \in [0,1]$ defines the rate of the weight decay. For standard SGD this is equivalent to adding an $L^2$-regularization term to the loss function $L$, see for example \cite{LoshchilovHutter2017}.

Motivated by the utility of weight regularization, we explore the efficacy of constraining the parameter space of Neural Networks to a suitable convex and compact region ${\mathcal C}$. The previously introduced methods would require a projection step during each update to maintain the feasibility of the parameters in this constrained setting. The standard SGD update would therefore become
\begin{equation}
	\theta_{t+1} = \Pi_{\mathcal C} \big( \theta_t - \alpha \tilde{\nabla} L(\theta_t) \big),
\end{equation}
where the projection function $\Pi_{\mathcal C}$ maps the input to its closest neighbor in the given feasible region ${\mathcal C}$ as measured by the $L^2$-norm. Depending on the particular feasible region, each projection step can be very costly, as often no closed expression is known and a separate optimization problem needs to be solved. We will instead explore a more appropriate alternative in the form of the Frank--Wolfe algorithm \citep{FrankWolfe1956}, also referred to as the Conditional Gradient algorithm \citep{LevitinPolyak1966}, a simple projection-free first-order algorithm for constrained optimization. In particular, we will be interested in applying stochastic variants of this algorithm. Rather than relying on a projection step like Gradient Descent methods, Frank--Wolfe algorithms instead call a linear minimization oracle (LMO) to determine
\begin{equation} \label{eq:LMO}
	v_t = \argmin_{v \in \mathcal C} \langle \tilde{\nabla} L(\theta_t), v \rangle,
\end{equation}
and move in the direction of $v_t$ through the update
\begin{equation} \label{eq:FWstep}
	\theta_{t+1} = \theta_t + \alpha ( v_t - \theta_t),
\end{equation}
where $\alpha \in [0,1]$. Feasibility is maintained, assuming that $\theta_0$ was initialized to lie in $\mathcal C$, since the update step consists of determining the convex combination of two points in the convex feasible region. Effectively, Frank--Wolfe algorithms minimize a linear first-order approximation of the loss function $L$ over the feasible region $\mathcal C$ and update the parameters by moving them closer towards the result of that linear optimization problem. These algorithms are therefore able to avoid the projection step by relying on an often computationally much cheaper LMO over $\mathcal C$.

We will demonstrate that Frank--Wolfe algorithms are viable candidates for training Neural Networks with constrained weights and that, in combination with an appropriately chosen region, the resulting networks achieve state-of-the-art test accuracy. We will also discuss the possibility of using specific constraints to achieve particular effects in the parameters of the networks, such as training sparse networks by using feasible regions spanned by sparse vectors. Lastly, we will show that different algorithms and feasible regions impact training and generalization behavior.

\paragraph{Related Work.} Frank--Wolfe algorithms have been well studied in the setting of smooth convex functions. Here \cite{HazanLuo2016} showed that the standard stochastic Frank--Wolfe algorithm (SFW) converges with a rate of $\mathcal{O}(1/t)$, assuming that the batch-sizes grows like $\Theta(t^2)$. Many variants have been proposed to improve the practical efficiency of SFW, most of these rely on modifying how the (unbiased) gradient estimator $\tilde{\nabla} L(\theta_t)$ is obtained: the Stochastic Variance-Reduced Frank--Wolfe algorithm (SVRF) \citep{HazanLuo2016} integrates variance reduction based on \cite{JohnsonZhang2013}, the Stochastic Path-Integrated Differential EstimatoR Frank--Wolfe algorithm (SPIDER-FW) \citep{YurtseverSraCevher2019, ShenFangZhaoHuangQian2019} integrates a different type of variance reduction based on \cite{FangLiChrisLinZhang2018} and the Online stochastic Recursive Gradient-based Frank--Wolfe algorithm (ORGFW) \citep{XieShenZhangQianWang2019} uses a form of momentum inspired by \cite{CutkoskyOrabona2019}. Related to these modifications of the gradient estimator, adding a momentum term to SFW has also been considered in several different settings and under many different names \citep{MokhtariHassaniKarbasi2020, MokhtariHassaniKarbasi2018, ChenHarshawhassaniKarbasi2018}. We will simply refer to this approach as SFW with momentum. For further related work regarding stochastic Frank--Wolfe methods, see \cite{LanZhou2016, LanPokuttaZhouZink2017, GoldfarbIyengarZhou2017, NegiarDresdnerTsaiGhaouiLocatelloPedregos2020, ZhangShenMokhtariHassaniKarbasi2020, CombettesPokuttaSpiegel2020}. So far however, little has been done to determine the \emph{practical real-world implications} of using stochastic Frank--Wolfe methods for the training of deep Neural Networks. One exception are the computational results by \cite{XieShenZhangQianWang2019}, which however are limited to fully connected Neural Networks with only a single hidden layer and therefore closer to more traditional setups. 

\paragraph{Contributions.} Our contribution is an inquiry into using projection-free methods for training Neural Networks and can be summarized as follows.

\medskip\noindent\emph{Achieving state-of-the-art test performance.} We demonstrate that stochastic Frank--Wolfe methods can achieve state-of-the-art test accuracy results on several well-studied benchmark datasets,  namely CIFAR-10, CIFAR-100, and ImageNet.

\medskip\noindent\emph{Constrained training affects learned features.} We show that the chosen feasible region significantly affects the encoding of information into the networks both through a simple visualization and by studying the number of active weights of networks trained on MNIST with various types of constraints.

\medskip\noindent\emph{Comparison of stochastic variants.} We compare different stochastic Frank--Wolfe algorithms and show that the standard SFW algorithm, as well as a variant that adds momentum, are both the most practical and best performing versions for training Neural Networks in terms of their generalization performance.

\paragraph{Outline.} We start by summarizing the necessary theoretical preliminaries regarding stochastic Frank--Wolfe algorithms as well as several candidates for feasible regions in Section~\ref{sec:preliminaries}. In Section~\ref{sec:technical-considerations} we will cover relevant technical considerations when constraining the parameters of Neural Networks and using stochastic Frank--Wolfe algorithms for training. Finally, in Section~\ref{sec:computational-results} we provide computational results. We conclude the paper with some final remarks in Section~\ref{sec:final-remarks}. Due to space limitations, all proofs and extended computational results have been relegated to the Appendix.

\section{Preliminaries} \label{sec:preliminaries}

We work in $(\mathbb{R}^n,\langle\cdot,\cdot\rangle)$, that is the Euclidean space with the standard inner product. We denote the $i$--th standard basis vector in $\mathbb R^n$ by $e_i$. 
The feasible regions $\mathcal{C}\subset\mathbb{R}^n$ we are interested in will be compact convex sets. For all $p\in\left[1,+\infty\right]$, let $\|\cdot\|_p$ denote the usual $L^p$-norm and $D = D(\mathcal C)\coloneqq\max_{x,y\in\mathcal{C}}\|y-x\|_2$ the $L^2$-diameter of $\mathcal{C}$. For every $i,j\in\mathbb{N}$ the double brackets $\llbracket i,j\rrbracket$ denote the set of integers between and including $i$ and $j$, assuming that $i\leq j$. For all $x\in\mathbb{R}^n$ and $i\in\llbracket1,n\rrbracket$, $[x]_i$ denotes the $i$-th entry of $x$.




\subsection{Stochastic Frank--Wolfe algorithms}
\label{sec:sfw}

We consider the constrained finite-sum optimization problem
	\begin{equation} \label{eq:finit-sum-problem}
		\min_{\theta \in \mathcal{C}} L(\theta) = \min_{\theta \in \mathcal{C}} \frac{1}{m} \sum_{i=1}^m \ell_i(\theta),
	\end{equation}
where the $\ell_i$ and therefore $L$ are differentiable in $\theta$ but possibly non-convex. Problems like this are at the center of Machine Learning, assuming the usual conventions regarding ReLU activation functions. We will denote its globally optimal solution by $\theta^\star$.

The pseudo-code of the standard stochastic Frank--Wolfe algorithm for this problem is stated in Algorithm~\ref{alg:sfw}. The random sample in Line~\ref{line:batch_sample} ensures that $\tilde{\nabla}L(\theta_t)$ in Line~\ref{line:gradient} is an unbiased estimator of $\nabla L(\theta_t)$, that is $\mathbb E \tilde{\nabla}L(\theta_t) = \nabla L(\theta_t)$. We have also included the option of applying a momentum term to the gradient estimate in Line~\ref{line:momentum} as is often done for SGD in Machine Learning. The algorithm also assumes access to a linear optimization oracle over the feasible region $\mathcal{C}$ that allows one to efficiently determine $\argmin_{v\in\mathcal{C}}\,\langle\tilde{\nabla}L(\theta_t),v\rangle$ in Line~\ref{line:lmo}. The update by convex combination in Line~\ref{line:update} ensures that $\theta_{t+1}\in\mathcal{C}$.

\begin{algorithm}[h]
\caption{Stochastic Frank--Wolfe (SFW)}
\label{alg:sfw}
\textbf{Input:} Initial parameters $\theta_0\in\mathcal{C}$, learning rate $\alpha_t \in \left[0,1\right]$, momentum $\rho_t \in \left[0,1\right]$, batch size $b_t \in \llbracket1, m \rrbracket$, number of steps $T$.\\
\vspace{-4mm}
\begin{algorithmic}[1]
\STATE $m_0 \leftarrow 0$ \label{line:momentum_init}
\FOR{$t=0$ \textbf{to} $T-1$}
\STATE uniformly sample i.i.d. $i_1, \ldots, i_{b_t}$ from $\llbracket 1, m \rrbracket$\label{line:batch_sample}
\STATE $\tilde{\nabla}L(\theta_t) \leftarrow \frac{1}{b_t} \sum_{j=1}^{b_t} \nabla \ell_{i_{j}} (\theta_t)$ \label{line:gradient}
\STATE $m_t \leftarrow (1-\rho_t) \, m_{t-1} + \rho_t \, \tilde{\nabla}L(\theta_t)$ \label{line:momentum}
\STATE$v_t \leftarrow \argmin_{v\in\mathcal{C}}\,\langle m_t,v\rangle$\label{line:lmo}
\STATE$\theta_{t+1} \leftarrow \theta_t+\alpha_t(v_t-\theta_t)$\label{line:update}
\ENDFOR
\end{algorithmic}
\end{algorithm}

\cite{ReddiSraPoczosBarnabasSmola2016} presented a convergence result for SFW (without the momentum term) in the non-convex setting. We denote the $L^2$-diameter of $\mathcal C$ by $D$ and we further define the \emph{Frank--Wolfe Gap} as
\begin{equation} \label{eq:fw-gap}
	\mathcal G(\theta) = \max_{v \in \mathcal C} \langle v-\theta, -\nabla L(\theta) \rangle.
\end{equation}
Note that $\mathcal G(\theta) = 0$ if and only if $\theta$ is a first order criticality, so $\mathcal G$ replaces the norm of the gradient as our metric for convergence in the constrained setting. A proof of the following statement will be included in Appendix~\ref{app:proof}.

\begin{theorem}[\cite{ReddiSraPoczosBarnabasSmola2016}] \label{thm:sfw_convergence}
	Consider the setting of Problem~\eqref{eq:finit-sum-problem} and assume that 
	the $\ell_i$ are smooth. If $\rho_t = 1$, $\alpha_t = T^{-1/2}$ and $b_t = T$ for all $0 \leq t < T$ and if $\theta_a$ is chosen uniformly at random from $\{\theta_i : 0 \leq i < T \}$ as determined by Algorithm~\ref{alg:sfw}, then we have 
	\begin{equation*}
		\mathbb E \, \mathcal G(\theta_a) = \mathcal O \left( \frac{L(\theta_0) - L(\theta^\star)}{T^{1/2}} \right), 
	\end{equation*}
	where $\mathbb E$ denotes the expectation w.r.t. all the randomness present.
\end{theorem}

The main focus of this paper will be on using both SFW and its momentum version to train deep Neural Networks. They are both straight forward to implement, assuming access to an LMO, and the results in the later sections will demonstrate their efficacy in this setting. We will however also include comparisons to some previously mentioned variants of Algorithm~\ref{alg:sfw}, that is in particular SVRF, SPIDER-FW, and ORGFW. Their pseudocode is included in Appendix~\ref{app:further-codes} and for convergence statements we refer to their respective papers.

\subsection{Regularization via feasible regions}

By imposing constraints on the parametrization of the Neural Network, we aim to control the structure of its weights and biases. In this part we will introduce relevant regions, discuss the associated linear minimization oracles, and also, for completeness, state the $L^2$-projection methods $\Pi_{\mathcal C}$ where appropriate. Note that Euclidean projection is not always the correct choice for all Gradient Descent methods, in particular adaptive optimizers like Adagrad and Adam require a more complicated projection based on the norm associated with previously accumulated gradient information.

The actual effects and implications of constraining the parameters of a Neural Network with any of these particular regions will be discussed in the next section. Note that, if the convex region $\mathcal{C}$ is a polytope, that is it is spanned by a finite set of vertices, the output of the LMO can always assumed to be one of these vertices.
See also Section~\ref{sec:technical-considerations} for further remarks on both of these aspects. For a list containing further potential candidates for feasible regions, see for example \cite{Jaggi2013}.

\paragraph{$L^p$-norm ball.}

The $L^p$-norm ball $\mathcal{B}_p(\tau) = \{x \in \mathbb{R}^n : \|x\|_p \leq \tau\}$ is convex for any $p \in [1, +\infty]$ and radius $\tau > 0$. The $L^2$-diameter of the $L^1$-norm ball and the $L^\infty$-norm ball, more commonly referred to as the \emph{hypercube}, are respectively given by $D(\mathcal{B}_1(\tau)) = 2\tau$ and $D(\mathcal{B}_\infty(\tau)) = 2 \tau \sqrt{n}$. For general $p \in (1, +\infty)$ we have $D(\mathcal{B}_p(\tau)) = 2 \tau n^{1/2 - 1/p}$.

\noindent \emph{LMO.} When $p \in (1, +\infty)$, the LMO over $\mathcal{B}_p(\tau)$ is given by
\begin{equation}
	\argmin_{v \in \mathcal{B}_p(\tau)} \langle v,x \rangle = -\tau \, \textrm{sgn}(x) |x|^{q / p} / \|x\|_q^{q/p},
\end{equation}
where $q$ is the complementary order to $p$ fulfilling $\nicefrac{1}{p} + \nicefrac{1}{q}  = 1$. For $p = 1$ and $p = \infty$ the oracle is given by the respective limits of this expression, i.e.,
\begin{equation}
	\argmin_{v \in \mathcal{B}_\infty(\tau)} \langle v,x \rangle = -\tau \, \textrm{sgn}(x),
\end{equation}
and
\begin{equation}
	[\argmin_{v \in \mathcal{B}_1(\tau)} \langle v,x \rangle]_i = \begin{cases} -\tau \, \textrm{sgn}([x]_i)  &\mbox{if } i = \argmax(|x|), \\ 
	0 & \mbox{otherwise}, \end{cases}
\end{equation}
that is the vector with a single non-zero entry equal to $-\tau \, \textrm{sign}(x)$ at a point where $|x|$ takes its maximum. Note that $\argmin_{v \in \mathcal{B}_1(\tau)} \langle v,x \rangle$ has a unique solution only if the entries of $|x|$ have a unique maximum. 

\noindent \emph{Projection.} If $x$ already lies in the feasible region, that is $x \in \mathcal B_p(\tau)$, then clearly $\Pi_{\mathcal C} (x) = x$. For $x \in \mathbb R^n \setminus \mathcal{B}_2(\tau)$, its $L^2$-projection into $\mathcal{B}_2(\tau)$ is given by
\begin{equation}
	\Pi_{\mathcal{B}_2(\tau)} (x) = \argmin_{v \in \mathcal{B}_2(\tau)} \|v - x\|_2 = \tau \, x / \|x\|_2.
\end{equation}
The $L^2$-projection onto $\mathcal{B}_\infty(\tau)$ for some given $x \in \mathbb R^n \setminus \mathcal{B}_\infty(\tau)$ is given by clipping the individual entries of $x$ to lie in $[-\tau,\tau]$, that is
\begin{equation}
	[\argmin_{v \in \mathcal{B}_1(\tau)} \|v - x\|_2]_i = \max( \min([x]_i, \tau), -\tau),
\end{equation}
for all $i \in \llbracket 1, n \rrbracket$. There are also algorithms capable of exact $L^2$-projections into $\mathcal{B}_1(\tau)$, see for example \cite{DuchiSHalevSingerChandra2008} for an algorithm of complexity $\mathcal O (n)$, but for general $p$ the projection task, unlike the LMO, poses a non-trivial (sub)optimization problem of its own.

\paragraph{\(K\)-sparse polytope.} For a fixed integer $K \in \llbracket 1, n \rrbracket$, the \(K\)-sparse polytope of radius $\tau > 0$ is obtained as the intersection of the $L^1$-ball $\mathcal{B}_1(\tau K)$ and the hypercube $\mathcal{B}_\infty (\tau)$. Equivalently, it can be defined as the convex hull spanned by all vectors in $\mathbb{R}^n$ with exactly $K$ non-zero entries, each of which is either $-\tau$ or $+\tau$. For $K = 1$ one recovers the $L^1$-norm ball and for $K = n$ the hypercube. The $L^2$-diameter of the \(K\)-sparse polytope of radius $\tau$ is given by $2 \tau \sqrt{K}$ assuming $n \geq K$.

\noindent \emph{LMO.} A valid solution to Equation~\eqref{eq:LMO} for the $K$-sparse polytope is given by the vector with exactly $K$ non-zero entries at the coordinates where $|x|$ takes its $K$ largest values, each of which is equal to $-\tau \, \textrm{sign}(x)$.

\paragraph{$K$-norm ball.}  For a fixed integer $K \in \llbracket 1, n \rrbracket$, the \(K\)-norm ball of radius $\tau$ can be defined as the convex hull of the union of the $L^1$-norm ball $\mathcal B_1(\tau)$ and the hypercube $\mathcal B_\infty (\tau / K)$. The \(K\)-norm was introduced in \cite{Watson1992}. For $K = 1$ one recovers the hypercube and for $K = n$ the $L^1$-norm ball. It is also the norm ball induced by the $K$-norm which is defined as the sum of the largest $K$ absolute entries in a vector. Its $L^2$-diameter is given by $\max ( 2\tau, 2 \tau \sqrt{n} / K )$.

\noindent \emph{LMO.} A valid solution to Equation~\eqref{eq:LMO} for the $K$-norm ball is easily obtained by taking the minimum of the LMOs of the $L^1$-norm ball of radius $\tau$ and the hypercube of radius $\tau K$.

\paragraph{Unit simplex.} The $\tau$-scaled unit simplex is defined by $\{\theta : \theta_1 + \ldots + \theta_n \leq \tau, \theta_i \geq 0\}$.  It can also be seen as the $n$-dimensional simplex spanned by all scaled standard basis vectors $\tau e_i$ in $\mathbb R^n$ as well as the zero vector. Its $L^2$-diameter is $\tau \sqrt{2}$.

\noindent \emph{LMO.} A valid solution to Equation~\eqref{eq:LMO} for the probability simplex is given by $\tau e_{i_0}$ where $i_0 = \argmin |x_i|$ if $x_{i_0} < 0$ and the zero vector otherwise.

\noindent \emph{Projection.} See \cite{ChenYe2011}.

\paragraph{Probability simplex.} The $\tau$-scaled probability simplex is defined as $\{x : x_1 + \ldots + x_n = \tau, x_i \geq 0\}$, that is all probability vectors in $\mathbb R^n$ multiplied by a factor of $\tau$. It can equivalently as the $(n-1)$-dimensional simplex spanned by all vectors $\tau e_i$ in $\mathbb R^n$. The $L^2$-diameter of the probability simplex is $\tau \sqrt{2}$.

\noindent \emph{LMO.} A valid solution to Equation~\eqref{eq:LMO} for the probability simplex is given by $\tau e_{i_0}$ where $i_0 = \argmin x_i$.

\noindent \emph{Projection.} 
See~\cite{WangCarreira2013}.

\paragraph{Permutahedron.} The permutahedron is the $(n-1)$-dimensional polytope spanned by all permutations of the coordinates of the vector $(1,2,\ldots,n)$. Its $L^2$-diameter is given by $( 2 k (k+1) (2k+1) / 6 )^{1/2}$ where $k = \lfloor n/2 \rfloor$.

\noindent \emph{LMO.} A solution to Equation~\eqref{eq:LMO} for the permutahedron can be obtained in polynomial time through the Hungarian method.

\noindent \emph{Projection.} See \cite{YasutakeHatanoKoheiKijimaTakimotoTakeda2011} and \cite{LimWright2016}.


\section{Technical Considerations} \label{sec:technical-considerations}


\subsection{Frank--Wolfe algorithms}

In the previous sections we referenced several stochastic variants of the Frank--Wolfe algorithm, namely SFW with and without momentum, SVRF, SPIDER-FW and ORGFW. As previously already stated, assuming the existence of an LMO, SFW is straight-forward to implement, both with and without momentum. Implementing SVRF, SPIDER-FW, and ORGFW however requires more care, since due to their variance reduction techniques, they require storing and using two or three different sets of parameters for the model and re-running batches with them two or even three times in an epoch. As such, the same kind of considerations apply as for their Gradient Descent equivalents, most prominently the need to keep any kind of randomness, e.g., through data augmentation and dropouts, fixed within each reference period. In the context of Gradient Descent methods, it has been suggested that these same techniques as suggested by \cite{JohnsonZhang2013} and \cite{CutkoskyOrabona2019} are not well suited to the context of Deep Learning, that is they offer little to no benefit to make up for the increase in complexity, see \cite{DefazioBottou2019}. In Section~\ref{sec:computational-results} we will computationally explore and confirm that the same reasoning applies to Frank--Wolfe algorithms.

Another important aspect is that we will treat all hyperparameters of the Frank--Wolfe algorithms, most notably the batch size, learning rate, and momentum parameter, as constant within a given run (unless a scheduler is specifically added) and to be tuned manually. Batch sizes can be chosen as is commonly done for Gradient Descent methods and momentum can likewise be set to $0.9$, that is $\rho = 0.1$. To make tuning of the learning rate easier, we have found it advantageous to at least partially decouple it from the size of the feasible region by dividing it by its $L^2$-diameter, that is Line~\ref{line:update} in Algorithm~\ref{alg:sfw} becomes
\begin{equation} \label{eq:decouple_diameter}
	\theta_{t+1} \leftarrow \theta_t + \min( \alpha   / D(\mathcal C), 1) \,  (v_t - \theta_t).
\end{equation}
This is similar to how the learning rate is commonly decoupled from the weight decay parameter for SGD~\citep{LoshchilovHutter2017}.

Another option to achieve a similar effect is to rescale the update vector $v_t - \theta_t$ to be of equal length as the gradient, that is Line~\ref{line:update} in Algorithm~\ref{alg:sfw} becomes
\begin{equation} \label{eq:gradient_rescale}
	\theta_{t+1} \leftarrow \theta_t + \min \left( \frac{  \alpha \| \tilde{\nabla} (\theta_t) \|}{ \|v_t - \theta_t\|_2}, 1 \right) (v_t - \theta_t) .
\end{equation}
Not only does this equally decouple the learning rate from the size of the particular region, but it also makes direct comparisons between Frank--Wolfe algorithms and Gradient Descent methods easier. It experimentally also seems to have a stabilizing effect on the algorithm when training very deep Neural Networks.

\subsection{Feasible regions}

In the previous sections we have expressed the constraints posed on all parameters in a given Neural Network through a single feasible region $\mathcal C$. In practice, one can probably limit oneself to the case of individual constraints placed on the parameters, or even just specific parts of those parameters, of individual layers of the network. Considering for example a simple, fully connected multilayer perceptron with $k$ layers, where the $i$-th layer $L_i$ consists of applying the operation $L_i(x) = a_i(W_ix + b_i)$ for some weight matrix $W_i$, a bias vector $b_i$ and a non-linear activation function $a_i$, one would require that $W_i$ lies in some feasible region $\mathcal C_{i,0}$ and $b_i$ in another region $\mathcal C_{i, 1}$.

For the purposes of demonstrating the general feasibility of constrained optimization of Neural Networks through the Frank--Wolfe algorithm, we have limited ourselves to uniformly applying the same type of constraint, such as a bound on the $L^p$-norm, separately on the weight and bias parameters of each layer, varying only the diameter of that region. We have found that linking the diameter of each feasible region to the expected initialization values performs well in practice, both for initializations as suggested by \cite{GlotoBengio2010} and \cite{HeZhangRenSun2015}. More specifically, if some weight vector $x \in \mathbb R^n$ is randomly initialized according to a zero-mean normal distribution with standard deviation $\sigma$, its expected $L^2$-norm is given by
\begin{equation}
	\mathbb E (\|x\|_2)	=  \frac{n \, \sigma \,  \Gamma (n/2 + 1/2)}{\sqrt{2} \, \Gamma (n/2 + 1) },
\end{equation}
and the diameter of a feasible region $\mathcal C$ would be determined by some fixed width $w > 0$ times that value, e.g., its radius $\tau$ may be chosen such that its $L^2$-diameter satisfies
\begin{equation} \label{eq:initialization_diameter}
	D(\mathcal C) = 2w\mathbb E (\|x\|_2).
\end{equation}
This adjusts the regularizing effect of the constraints to the particular diameter of the region.
 
Going forward, we believe however that using more tailor-made constraints chosen to suit the type and size of each layer can be of great interest and provide additional tools in the design and training of Neural Networks. To that extend, it should be noted that the particular choice of feasible region has a strong impact on the learned features that goes beyond the regularizing effect of constraining the features to not be `too large`, see also the computational results in the following section. This is independent of convergence guarantees such as Theorem~\ref{thm:sfw_convergence}. In particular, the solution given by the LMO impacts the update of the parameters in Algorithm~\ref{alg:sfw}.

As part of this observation, we also note that when stating LMOs in the previous section, we have implicitly assumed the existence of a unique solution. In particular, this means that the gradient was non-zero, though depending on the feasible region other gradients can also lead to non-unique solutions to Equation~\eqref{eq:LMO}. We suggest settling these cases by randomly sampling a valid solution that lies on the boundary of the feasible region. In the particular case of polytopes, one can sample randomly among all vertices that pose a valid solution. While this remark might at first glance seem not relevant in practice, as the usual combination of randomly initialized weights and Gradient Descent based optimization generally avoids scenarios that might lead to non-unique solutions to Equation~\eqref{eq:LMO}, we remark that the stochastic Frank--Wolfe algorithm is for example capable of training a zero-initialized layer that is constrained to a hypercube when the previous remarks are taken into account.

%



\paragraph{$L^2$-norm ball.} Constraining the $L^2$-norm of weights and optimizing them using the Frank--Wolfe algorithm is most comparable, both in theory and in practice, to the well-established optimization of unconstrained parameters through SGD with an $L^2$-regularization term added to the cost function. Note that the output to the LMO is parallel to the gradient, so as long as the current iterate of the weights is not close to the boundary of the $L^2$-norm ball, the update of the SFW algorithm $\theta_{t+1} \leftarrow \theta_t+\alpha(v_t-\theta_t)$ is similar to that of SGD.

\paragraph{Hypercube.} Requiring each individual weight of a network or a layer to lie within a certain range, say in $[-\tau,\tau]$, is possibly an even more natural type of constraint. Here the update step taken by the Frank--Wolfe algorithm however differs drastically from that taken by projected SGD: in the output of the LMO each parameter receives a value of equal magnitude, so to a degree all parameters are forced to receive a non-trivial update each step.

\paragraph{$L^1$-norm ball and $K$-sparse polytopes.} On the other end of the spectrum from the dense updates forced by the LMO of the hypercube are feasible regions whose LMOs return very sparse vectors, e.g., the $L^1$-norm ball and its generalization, $K$-sparse polytopes. When for example constraining the $L^1$-norm of weights of a layer, only a single weight, that from which the most gain can be derived, will in fact increase in absolute value during the update step of the Frank--Wolfe algorithm while all other weights will decay and move towards zero. The $K$-sparse polytope generalizes that principle and increases the absolute value of the $K$ most important weights. This has the potential of resulting in very sparse weight matrices which have been of recent interest, see for example \citep{EvciGaleMenickCastroElsen2019, SrinivasSubramanyaAkshayvarunVenkatesh2017,LouizosWellingKingma2017}.

%

\section{Computational Results} \label{sec:computational-results}

\subsection{Comparing Frank--Wolfe algorithms}

We compare the performance of four different variants of the stochastic Frank--Wolfe algorithm introduced in Sections~\ref{sec:introduction} and~\ref{sec:preliminaries}, namely SFW both with and without momentum (the latter will be abbreviated as MSFW in the graphs) as well as SVRF and ORGFW. Note that the pseudo-codes of SFRV and ORGFW are stated in Appendix~\ref{app:further-codes} and that SPIDER-FW was omitted from the results presented here as it was not competitive on these specific tasks.

For the first comparison, we train a fully connected Neural Network with two hidden layers of size $64$ on the Fashion-MNIST dataset \citep{fashion_mnist} for $10$ epochs. The parameters of each layer are constrained in their $L^1$-norm as suggested in Section~\ref{sec:technical-considerations} and all hyperparameters are tuned individually for each of the algorithms. The results, averaged over $10$ runs, are presented in Figure~\ref{fig:fw_comparison_vr-fashion_mnist-simple}.

\begin{figure}[h]
\centerline{\includegraphics[trim=8 8 8 10, clip, width=0.5\columnwidth]{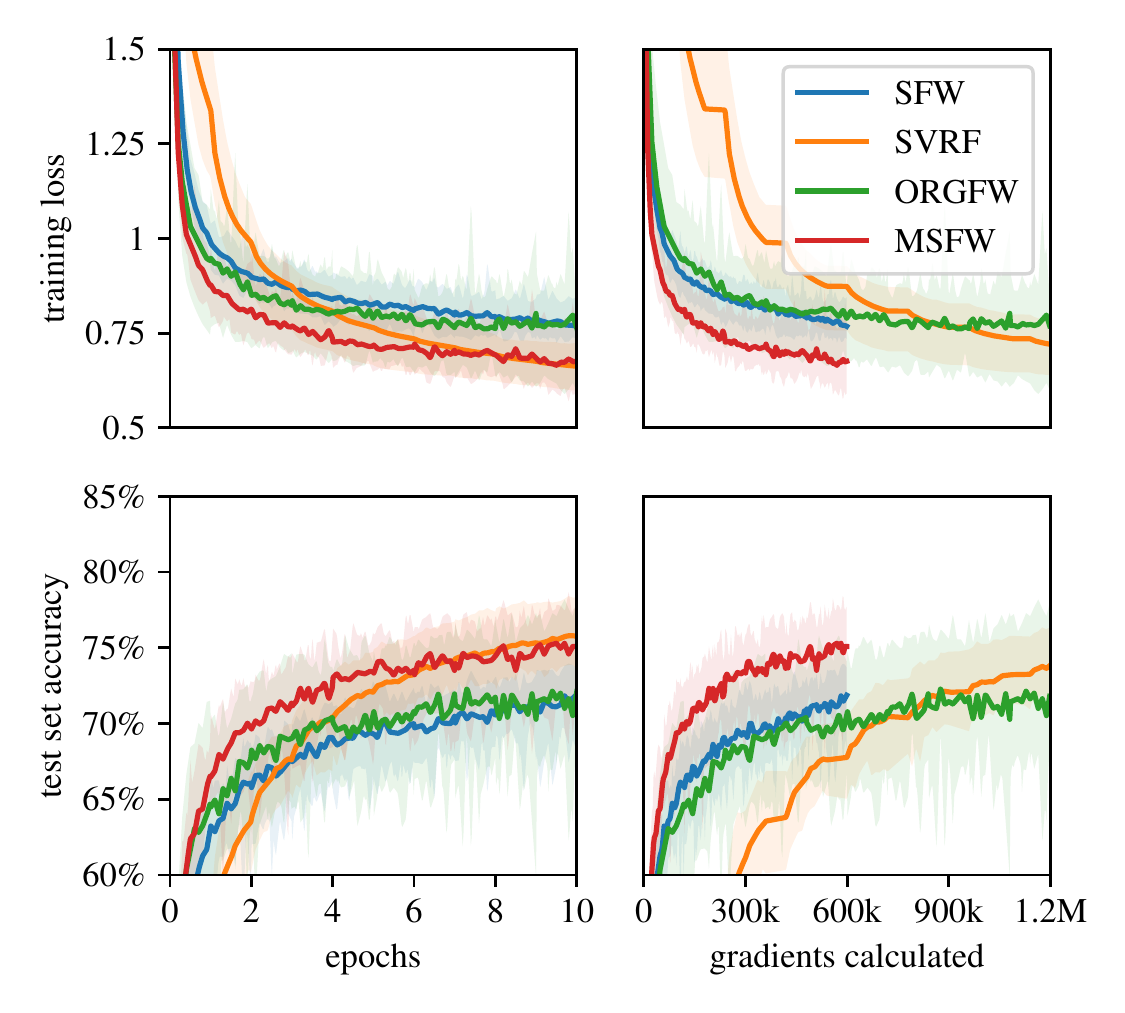}}
\caption{Comparing different stochastic Frank--Wolfe algorithms to train a fully connected Neural Network with two hidden layers on the Fashion-MNIST dataset.}
\label{fig:fw_comparison_vr-fashion_mnist-simple}
\end{figure}

For the second comparison, we train a fully-connected Neural Network with one hidden layer of size $64$ on the IMDB dataset of movie reviews \citep{imdb} for $10$ epochs. We use the \href{https://www.tensorflow.org/datasets/catalog/imdb_reviews#imdb_reviewssubwords8k}{8\,185 subword representation from TensorFlow} to generate sparse feature vectors for each datapoint. The parameters of each layer are constrained in their $L^\infty$-norm. The results, averaged over $10$ runs per algorithm, are presented in Figure~\ref{fig:fw_comparison_vr-imdb-simple}.

\begin{figure}[h]
\centerline{\includegraphics[trim=8 8 8 10, clip, width=0.5\columnwidth]{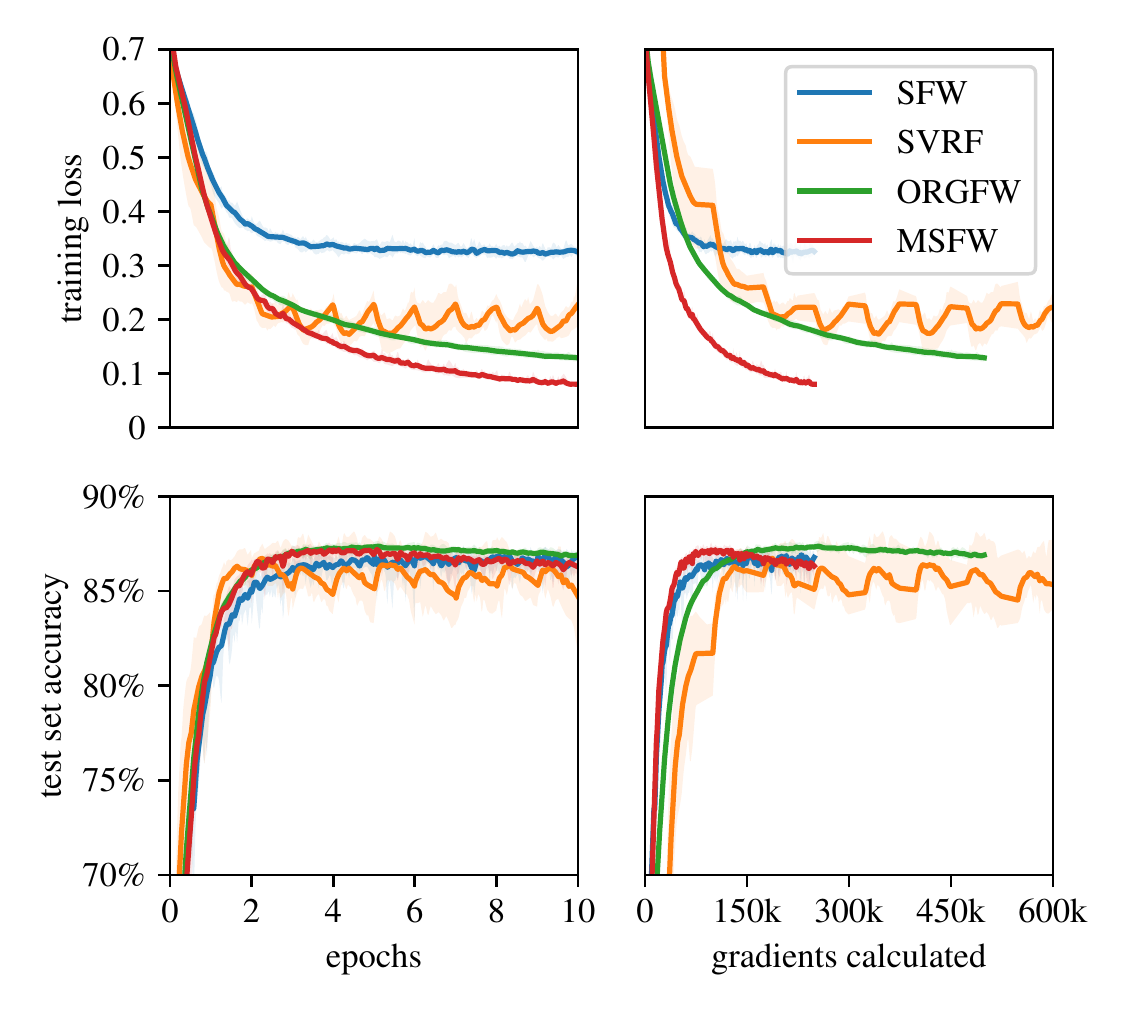}}
\caption{Comparing different stochastic Frank--Wolfe algorithms to train a fully connected Neural Network with one hidden layers on sparse feature vectors generated from the IMDB dataset of movie reviews.}
\label{fig:fw_comparison_vr-imdb-simple}
\end{figure}

Considering these two comparisons, it can be observed that in some scenarios both SVRF and ORGFW can provide an increase in performance w.r.t. SFW both on the train and the test set when considering the relevant metrics vs. the number of epochs, i.e., passes through the complete dataset. However, SFW with momentum (MSFW), a significantly easier algorithm to implement and use in actual Deep Learning applications, provides a comparable boost in epoch performance and significantly improves upon it when considering the metrics vs. the number of stochastic gradient evaluations, a more accurate metric of the involved computational effort. This is due to the variance reduction techniques used in algorithms like SVRF and ORGFW requiring multiple passes over the same datapoint with different parameters for the model within one epoch. It was previously already suggested, e.g., by \cite{WilsonRoelofsSternSrebroRecht2017}, that these techniques offer little benefit for Gradient Descent methods used to train large Neural Networks. Based on this and the computations presented here, we will focus on SFW, both with and without momentum, for the remainder of the computations in this section. For further results and complete setups, see Appendix~\ref{app:computations}.

\begin{table*}[ht]
\label{table:stoa}
\begin{center}
\begin{tabular}{lccccc}
\toprule
& \multicolumn{2}{c}{\bf CIFAR-10} &  \multicolumn{1}{c}{\bf CIFAR-100} & \multicolumn{2}{c}{\bf ImageNet} \\
& DenseNet121 & WideResNet28x10 & GoogLeNet & DenseNet121 & ResNeXt50 \\
\midrule
\bf SGD \tiny{without weight decay} & 93.14\% \tiny{\textpm 0.11} & 94.44\% \tiny{\textpm 0.12} & 76.82\% \tiny{\textpm 0.25} & 71.06\%  & 70.15\% \\
\bf SGD \tiny{with weight decay} & 94.01\% \tiny{\textpm 0.09} & \bf 95.13\% \tiny{\textpm 0.11}  & 77.50\% \tiny{\textpm 0.13} & \bf 74.89\% & \bf 76.09\%  \\
\bf SFW \tiny{with $L^2$-constraints} & \bf 94.46\% \tiny{\textpm  0.13} & 94.58\% \tiny{\textpm 0.18} & \bf 78.88\% \tiny{\textpm 0.10} & 73.46\% & 75.77\%  \\
\bf SFW \tiny{with $L^\infty$-constraints} & 94.20\% \tiny{\textpm 0.19} & 94.03\% \tiny{\textpm 0.35} & 76.54\% \tiny{\textpm 0.50} & 72.22\% & 73.95\% \\
\bottomrule
\end{tabular}
\end{center}
\caption{Test accuracy attained by several deep Neural Networks trained on the CIFAR-10, CIFAR-100 and ImageNet datasets. Parameters trained with SGD were unconstrained. Full results can be found in Appendix~\ref{app:computations}.}
\end{table*}

\subsection{Visualizing the impact of constraints} \label{sec:mnist-visulization}

We will next illustrate the impact that the choice of constraints has on the learned representations through a simple classifier trained on the MNIST dataset \citep{mnist}, which consists of $28 \times 28$ pixel grayscale images of handwritten digits. The particular network chosen here, for the sake of exposition, represents a linear regression, i.e., it has no hidden layers and no bias terms and the flattened input layer of size $784$ is fully connected to the output layer of size $10$. The weights of the network are therefore represented by a single $784 \times 10$ matrix, where each of the ten columns corresponds to the weights learned to recognize the ten digits $0$ to $9$. In Figure~\ref{fig:mnist_visualization} we present a visualization of this network trained on the dataset with different types of constraints placed on the parameters. Each image interprets one of the columns of the weight matrix as an image of size $28 \times 28$ where red represents negative weights and green represents positive weights for a given pixel. We see that the choice of feasible region, and in particular the LMO associated with it, can have a drastic impact on the representations learned by the network when using stochastic Frank--Wolfe algorithms. This is in line with the observations stated in Section~\ref{sec:technical-considerations}. For a complete visualization including other types of constraints and images, see Figures~\ref{fig:mnist_visualization_complete} and~\ref{fig:fashion_mnist_visualization_complete} in Appendix~\ref{app:computations}.

\begin{figure}[h]
\centerline{\includegraphics[trim=10 10 10 10, clip, width=0.5\columnwidth]{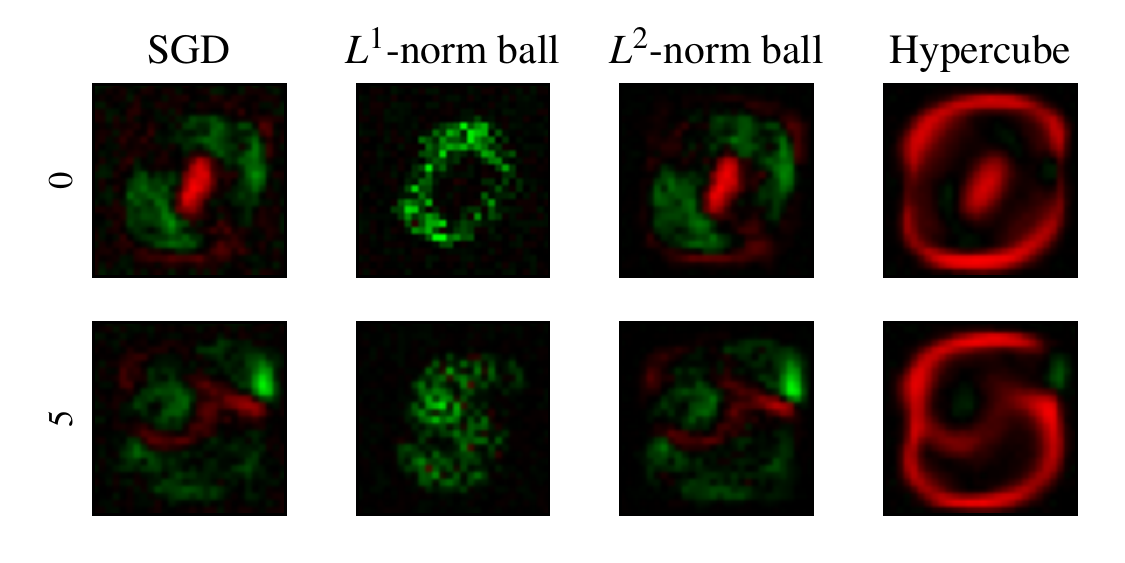}}
\caption{Visualization of the weights in a fully connected no-hidden-layer classifier trained on the MNIST dataset corresponding to the digits $0$ and $3$. Red corresponds to negative and green to positive weights. The unconstrained network in the first column was trained using SGD and the constrained networks in the remaining columns were trained using SFW.}
\label{fig:mnist_visualization}
\end{figure}

\subsection{Sparsity during training}

Further demonstrating the impact that the choice of a feasible region has on the learned representations, we consider the sparsity of the weights of trained networks. To do so, we consider the parameter of a network to be \emph{inactive} when it is smaller in absolute terms than its random initialization value. Using this notion, we can create sparse matrices from the weights of a trained network by setting all weights corresponding to inactive parameters to zero.

To study the effect of constraining the parameters, we trained two different types of networks, a fully connected network with two hidden layers with a total of $26,506$ parameters and a convolutional network with $93,322$, on the MNIST dataset. The weights of these networks were either unconstrained and updates performed through SGD, both with and without weight decay applied, or they were constrained to lie in a certain feasible region and trained using SFW. The results are shown in Figure~\ref{fig:sparseness_sparse-mnist}.

\begin{figure}[h]
\centerline{\includegraphics[trim=6 6 6 7, clip, width=0.5\columnwidth]{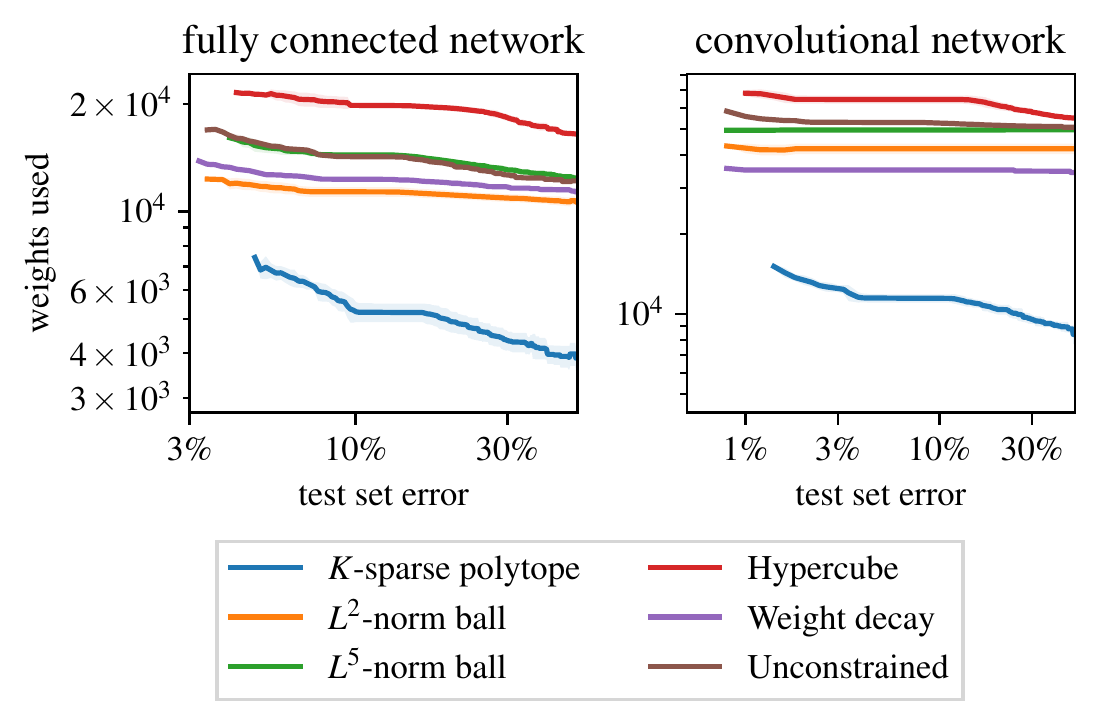}}
\caption{Number of active parameters vs. test set error in two different networks trained on the MNIST dataset. When the parameters are contrained to lie in a specified feasible region, the networks were trained using SFW with momentum. In the unconstrained case they were trained using SGD. Results are averaged over 5 runs. The full setup is in Appendix~\ref{app:computations}.}
\label{fig:sparseness_sparse-mnist}
\end{figure}

We see that regions spanned by sparse vectors, such as $K$-sparse polytopes, result in noticeably fewer active parameters in the network over the course of training, whereas regions whose LMO forces larger updates in each parameter, such as the Hypercube, result in more active weights. 

\subsection{Training very deep Neural Networks}

Finally, we demonstrate the feasibility of training even very deep Neural Networks using stochastic Frank--Wolfe algorithms. We have trained several state-of-the-art Neural Networks on the CIFAR-10, CIFAR-100 and ImageNet datasets~\citep{cifar, imagenet}. In Table~\ref{table:stoa} we show the top-1 test accuracy attained by networks based on the DenseNet, WideResNet, GoogLeNet and ResNeXt architecture on the test sets of these datasets. Here we compare networks with unconstrained parameters trained using SGD with momentum both with and without weight decay as well as networks whose parameters are constrained in their $L^2$-norm or $L^\infty$-norm, as laid out Section~\ref{sec:technical-considerations}, and which were trained using SFW with momentum. Both the weight decay parameter and the size of the feasible region was tuned individually and we scaled the updates of SFW to be comparable with that of SGD, again as laid out Section~\ref{sec:technical-considerations}. We can observe that, when constraining the $L^2$-norm of the parameters, SFW attains performance exceeding that of standard SGD and matching the state-of-the-art performance of SGD with weight decay. When constraining the $L^\infty$-norm of the parameters, SFW does not quite achieve the same performance as SGD with weight decay, but a regularization effect through the constraints is nevertheless clearly present, as it still exceeds the performance of SGD without weight decay. We furthermore note that, due to the nature of the LMOs associated with these particular regions, runtimes were comparable. 

\section{Final remarks} \label{sec:final-remarks}

The primary purpose of this paper is to promote the use of constraint sets to train Neural Networks even in state-of-the-art settings. We have developed implementations of the methods presented here both in TensorFlow~\citep{tensorflow} and in PyTorch~\citep{pytorch} and have made the code publicly available on \href{https://github.com/ZIB-IOL/StochasticFrankWolfe}{github.com/ZIB-IOL/StochasticFrankWolfe} both for reproducibility and to encourage further research in this area.


\subsubsection*{Acknowledgements}

Research reported in this paper was partially supported by the Research Campus MODAL funded by the German Federal Ministry of Education and Research (grant number 05M14ZAM). All computational results in this paper were tracked using the educational license of Weights \& Biases~\citep{wandb}.


\clearpage
\appendix
\onecolumn

\section{Proof of Theorem~\ref{thm:sfw_convergence}} \label{app:proof}

For a given probability space $\Omega$, we consider the stochastic optimization problem of the form 
	\begin{equation} \label{eq:stochastic-problem}
		\min_{\theta \in \mathcal{C}} L(\theta) = \min_{\theta \in \mathcal{C}} \mathbb{E}_{\omega \in \Omega} \ell(\theta, \omega),
	\end{equation}
where we assume that $L$ and $\ell$ are differentiable in $\theta$ but possibly non-convex and $\mathcal C$ is convex and compact. We slightly reformulate Algorithm~\ref{alg:sfw} (without momentum) for this stochastic problem formulation in Algorithm~\ref{alg:sfw-stoch}.

\begin{algorithm}[H]
\caption{Stochastic Frank--Wolfe (SFW)}
\label{alg:sfw-stoch}
\textbf{Input:} Initial parameters $\theta_0\in\mathcal{C}$, learning rate $\alpha_t \in \left[0,1\right]$, batch size $b_t$, number of steps $T$.\\
\vspace{-4mm}
\begin{algorithmic}[1]
\FOR{$t=0$ \textbf{to} $T-1$}
\STATE sample i.i.d. $\omega_1^{(t)}, \ldots, \omega_{b_t}^{(t)} \in \Omega$ \label{line:batch_sample_stoch}
\STATE $\tilde{\nabla}L(\theta_t) \leftarrow \frac{1}{b_t} \sum_{j=1}^{b_t} \nabla \ell(\theta_t, \omega_j^{(t)})$ \label{line:gradient_stoch}
\STATE$v_t \leftarrow \argmin_{v\in\mathcal{C}}\,\langle \tilde{\nabla}L(\theta_t),v\rangle$\label{line:lmo_stoch}
\STATE$\theta_{t+1} \leftarrow \theta_t+\alpha_t(v_t-\theta_t)$\label{line:update_stoch}
\ENDFOR
\end{algorithmic}
\end{algorithm}

Let us recall some definitions. We denote the globally optimal solution to Equation~\ref{eq:stochastic-problem} by $\theta^\star$ and the \emph{Frank--Wolfe Gap} is defined as
\begin{equation} \label{eq:fw-gap-2}
	\mathcal G(\theta) = \max_{v \in \mathcal C} \langle v-\theta, -\nabla L(\theta) \rangle.
\end{equation}

Let us assume that $\ell$ is $M$-smooth, that is
\begin{equation}
	\| \nabla \ell (x, \omega) - \nabla \ell (y, \omega) \| \leq M \|x - y\|,
\end{equation}
for any $\omega \in \Omega$. We note that it follows that $L$ is also $M$-smooth and furthermore
\begin{equation}
	G = \max_{x \in \mathcal C} \sup_{\omega \in \Omega} \| \nabla \ell (x, \omega) \|
\end{equation}
exists, that is $\ell$ is also $G$-Lipschitz.
%
A well known consequence of $L$ being $M$-smoothn is that 
\begin{equation}
	L(x) \leq L(y) + \langle \nabla L(y), x-y \rangle + \frac{M}{2}	\| x-y \|^2.
\end{equation}
Denote the $L^2$-diameter of $\mathcal C$ by $D (\mathcal C) = D$, that is $\|x - y\| \leq D$ for  all $x,y \in \mathcal C$. Let $\beta \in \mathbb{R}$ satisfy
\begin{equation}
	\beta \geq \frac{2(L(\theta_0) - L(\theta^\star))}{MD^2},
\end{equation}
for some given initialization $\theta_0 \in \mathcal C$ of the parameters. Under these assumptions, the following generalization of Theorem~\ref{thm:sfw_convergence} holds. Note that the dependency of the learning rate on $L(\theta_0) - L(\theta^\star)$ can be removed by simply setting $\beta$ equal to its lower bound.

\begin{theorem}[\cite{ReddiSraPoczosBarnabasSmola2016}] \label{thm:sfw_convergence_extended}
	Let $b_t = b = T$ and 
	\begin{equation}
		\alpha_t = \alpha = \left( \frac{L(\theta_0) - L(\theta^\star)}{TMD^2 \beta} \right)^{1/2},
	\end{equation}
	for all $0 \leq t < T$. If $\theta_a$ is chosen uniformly at random from $\{\theta_i : 0 \leq i < T \}$ as determined by Algorithm~\ref{alg:sfw-stoch} applied to Equation~\eqref{eq:stochastic-problem}, then we have 
	\begin{equation*}
		\mathbb E \, \mathcal G(\theta_a) \mathcal \leq \frac{D}{\sqrt{T}} \left( G + \left( \frac{2M (L(\theta_0) - L(\theta^\star))}{\beta} \right)^{1/2} (1+\beta) \right),
	\end{equation*}
	where $\mathbb E$ denotes the expectation w.r.t. all the randomness present.
\end{theorem}

To state a proof, we will need the following well-established Lemma quantifying how closely $\tilde{\nabla} L(\theta)$ approximates $\nabla (\theta_t)$. A proof can for example be found in \cite{ReddiSraPoczosBarnabasSmola2016}.

\begin{lemma} \label{lemma:gradient-approximation}
	Let $\omega_1, \ldots, \omega_b$ be i.i.d. samples in $\Omega$, $\theta \in \mathcal C$ and $\tilde{\nabla}L(\theta) = \frac{1}{b} \sum_{j=1}^{b} \nabla \ell(\theta_t, \omega_j)$. If $\ell$ is $G$-Lipschitz, then
	\begin{equation}
		\mathbb E \|\tilde{\nabla} L (\theta) - \nabla(\theta) \|\leq \frac{G}{b^{1/2}}.
	\end{equation}
\end{lemma}

\begin{proof}[Proof of Theorem~\ref{thm:sfw_convergence_extended}]
	By $M$-smoothness of $L$ we have
	\begin{align*}
		L(\theta_{t+1}) & \leq L(\theta_t)  + \langle \nabla L(\theta_t) ,  \theta_{t+1} - \theta_t \rangle + \frac{M}{2} \| \theta_{t+1} - \theta_t \|^2.
	\end{align*}
	Using the fact that $\theta_{t+1} = \theta_{t} + \alpha (v_t - \theta_t)$ in Algorithm~\ref{alg:sfw-stoch} and that $\| v_t - \theta_t \| \leq D$, it follows that 
	\begin{equation} \label{eq:smoothness-consequence}
		L(\theta_{t+1}) \leq L(\theta_t)  + \alpha \langle \nabla L(\theta_t) ,   v_t - \theta_t \rangle + \frac{MD^2\alpha^2}{2} .
	\end{equation}
	Now let
	\begin{equation}
		\hat{v}_t = \argmin_{v \in \mathcal C} \langle \nabla L (\theta_t), v \rangle
	\end{equation}
	for any $t = 0, \ldots, T-1$. Note the difference to the definition of $v_t$ in Line~\ref{line:lmo_stoch} in Algorithm~\ref{alg:sfw-stoch} and also note that
	\begin{equation}
		\mathcal G(\theta) = \max_{v \in \mathcal C} \langle v-\theta, -\nabla L(\theta) \rangle = \langle \hat{v}_t - \theta, -\nabla L(\theta).\rangle.
	\end{equation}
	Continuing from Equation~\eqref{eq:smoothness-consequence}, we therefore have
	\begin{align*} 
		L(\theta_{t+1}) & \leq L(\theta_t)  + \alpha \langle \tilde{\nabla} L(\theta_t) ,   v_t - \theta_t \rangle  + \alpha \langle \nabla L(\theta_t) - \tilde{\nabla} L(\theta_t) ,   v_t - \theta_t \rangle + \frac{MD^2\alpha^2}{2} \\
		& \leq L(\theta_t)  + \alpha \langle \tilde{\nabla} L(\theta_t) ,   \hat{v}_t - \theta_t \rangle  + \alpha \langle \nabla L(\theta_t) - \tilde{\nabla} L(\theta_t) ,   v_t - \theta_t \rangle + \frac{MD^2\alpha^2}{2} \\
		& = L(\theta_t)  + \alpha \langle \nabla L(\theta_t) ,   \hat{v}_t - \theta_t \rangle  + \alpha \langle \nabla L(\theta_t) - \tilde{\nabla} L(\theta_t) ,   v_t - \hat{v}_t \rangle + \frac{MD^2\alpha^2}{2} \\
		& = L(\theta_t)  - \alpha \, \mathcal G(\theta_t)  + \alpha \langle \nabla L(\theta_t) - \tilde{\nabla} L(\theta_t) ,   v_t - \hat{v}_t \rangle + \frac{MD^2\alpha^2}{2}
	\end{align*}
	Applying Cauchy--Schwarz and using the fact that the diameter of $\mathcal C$ is $D$, we therefore have
	\begin{equation} 
		L(\theta_{t+1}) \leq L(\theta_t)  - \alpha \, \mathcal G(\theta_t)  + \alpha D \| \nabla L(\theta_t) - \tilde{\nabla} L(\theta_t)\| + \frac{MD^2\alpha^2}{2}.
	\end{equation}
	Taking expectations and applying Lemma~\ref{lemma:gradient-approximation}, we get
	\begin{equation} 
		\mathbb E L(\theta_{t+1}) \leq \mathbb E L(\theta_t)  - \alpha \, \mathbb E \mathcal G(\theta_t)  + \frac{GD \alpha}{b^{1/2}} + \frac{MD^2\alpha^2}{2}.
	\end{equation}
	By rearranging and summing over $t = 0, \ldots, T-1$, we get the upper bound
	\begin{align} 
		\sum_{t=0}^{T-1} \mathbb E \mathcal G(\theta_t) & \leq \frac{L(\theta_0) -\mathbb E L(\theta_T)}{\alpha} + \frac{GD}{b^{1/2}} + \frac{MD^2\alpha}{2} \nonumber \\
		& \leq \frac{L(\theta_0) -\mathbb E L(\theta^\star)}{\alpha} + \frac{GD}{b^{1/2}} + \frac{MD^2\alpha}{2}.
	\end{align}
	Using the definition of $\theta_a$ as well as the specified values for $\alpha$ and $b$ and the bound for $\beta$, we get the desired
	\begin{equation*}
		\mathbb E \, \mathcal G(\theta_a) \mathcal \leq \frac{D}{\sqrt{T}} \left( G + \left( \frac{2M (L(\theta_0) - L(\theta^\star))}{\beta} \right)^{1/2} (1+\beta) \right).
	\end{equation*}
\end{proof}

\newpage

\section{Further stochastic Frank--Wolfe algorithms}  \label{app:further-codes}

\begin{algorithm}[H]
\caption{SVRF~\citep{HazanLuo2016}}
\label{alg:svrf}
\textbf{Input:} Initial parameters $\theta_0\in\mathcal{C}$, learning rate $\alpha_{t,k} \in \left[0,1\right]$, batch size $b_{t,k} \in \llbracket1, m \rrbracket$, number of steps $T$ and $K_t$.\\
\vspace{-4mm}
\begin{algorithmic}[1]
\FOR{$t=0$ \textbf{to} $T-1$}
\STATE take snapshot $x_0 = \theta_{t}$ and compute $\nabla L ( x_0 )$
\FOR{$k=0$ \textbf{to} $K_t$}
\STATE uniformly sample i.i.d. $i_1, \ldots, i_{b_{t,k}}$ from $\llbracket 1, m \rrbracket$
\STATE $\tilde{\nabla}L(x_k) \leftarrow \nabla L ( x_0 ) + \frac{1}{b_{t,k}} \sum_{j=1}^{b_{t,k}} \left( \nabla \ell_{i_{j}} (x_k) - \ell_{i_{j}} (x_0) \right)$ 
\STATE$v_k \leftarrow \argmin_{v\in\mathcal{C}}\,\langle \tilde{\nabla}L(x_k),v\rangle$
\STATE$x_{k+1} \leftarrow x_k+\alpha_{t,k} (v_k-x_k)$
\ENDFOR
\STATE $\theta_{t+1} \leftarrow x_{K_t}$
\ENDFOR
\end{algorithmic}
\end{algorithm}

\begin{algorithm}[H]
\caption{SPIDER-FW~\citep{YurtseverSraCevher2019, ShenFangZhaoHuangQian2019}}
\label{alg:spider-fw}
\textbf{Input:} Initial parameters $\theta_0\in\mathcal{C}$, learning rate $\alpha_{t,k} \in \left[0,1\right]$, batch size $b_{t,k} \in \llbracket1, m \rrbracket$, number of steps $T$ and $K_t$.\\
\vspace{-4mm}
\begin{algorithmic}[1]
\FOR{$t=0$ \textbf{to} $T-1$}
\STATE take snapshot $x_0 = \theta_{t}$ and compute $\tilde{\nabla} L ( x_0 ) \leftarrow \nabla L ( x_0 )$
\FOR{$k=1$ \textbf{to} $K_t$}
\STATE uniformly sample i.i.d. $i_1, \ldots, i_{b_{t,k}}$ from $\llbracket 1, m \rrbracket$
\STATE $\tilde{\nabla}L(x_k) \leftarrow \tilde{\nabla}L(x_{k-1}) + \frac{1}{b_{t,k}} \sum_{j=1}^{b_{t,k}} \left( \nabla \ell_{i_{j}} (x_k) - \ell_{i_{j}} (x_{\max(k-1,0)}) \right)$ 
\STATE$v_k \leftarrow \argmin_{v\in\mathcal{C}}\,\langle \tilde{\nabla}L(x_k),v\rangle$
\STATE$x_{k+1} \leftarrow x_k+\alpha_{t,k} (v_k-x_k)$
\ENDFOR
\STATE $\theta_{t+1} \leftarrow x_{K_t}$
\ENDFOR
\end{algorithmic}
\end{algorithm}

\begin{algorithm}[H]
\caption{ORGFW~\citep{XieShenZhangQianWang2019}}
\label{alg:orgfw}
\textbf{Input:} Initial parameters $\theta_0\in\mathcal{C}$, learning rate $\alpha_t \in \left[0,1\right]$, momentum $\rho_t \in \left[0,1\right]$, batch size $b_t \in \llbracket1, m \rrbracket$, number of steps $T$.\\
\vspace{-4mm}
\begin{algorithmic}[1]
\FOR{$t=0$ \textbf{to} $T-1$}
\STATE uniformly sample i.i.d. $i_1, \ldots, i_{b_t}$ from $\llbracket 1, m \rrbracket$
\IF{t = 0}
\STATE $m_0 \leftarrow \frac{1}{b_t} \sum_{j=1}^{b_t} \nabla \ell_{i_{j}} (\theta_0)$ 
\ELSE
\STATE $m_t \leftarrow \frac{1}{b_t} \sum_{j=1}^{b_t} \nabla \ell_{i_{j}} (\theta_t) + (1-\rho_t) \, \left( m_{t-1} - \frac{1}{b_t} \sum_{j=1}^{b_t} \nabla \ell_{i_{j}} (\theta_{t-1}) \right)$ 
\ENDIF
\STATE$v_t \leftarrow \argmin_{v\in\mathcal{C}}\,\langle m_t,v\rangle$
\STATE$\theta_{t+1} \leftarrow \theta_t+\alpha_t(v_t-\theta_t)$
\ENDFOR
\end{algorithmic}
\end{algorithm}

\newpage

\section{Further computational results and complete setups} \label{app:computations}

\begin{figure}[h]
\centerline{\includegraphics[trim=8 8 8 8, clip, width=0.8\columnwidth]{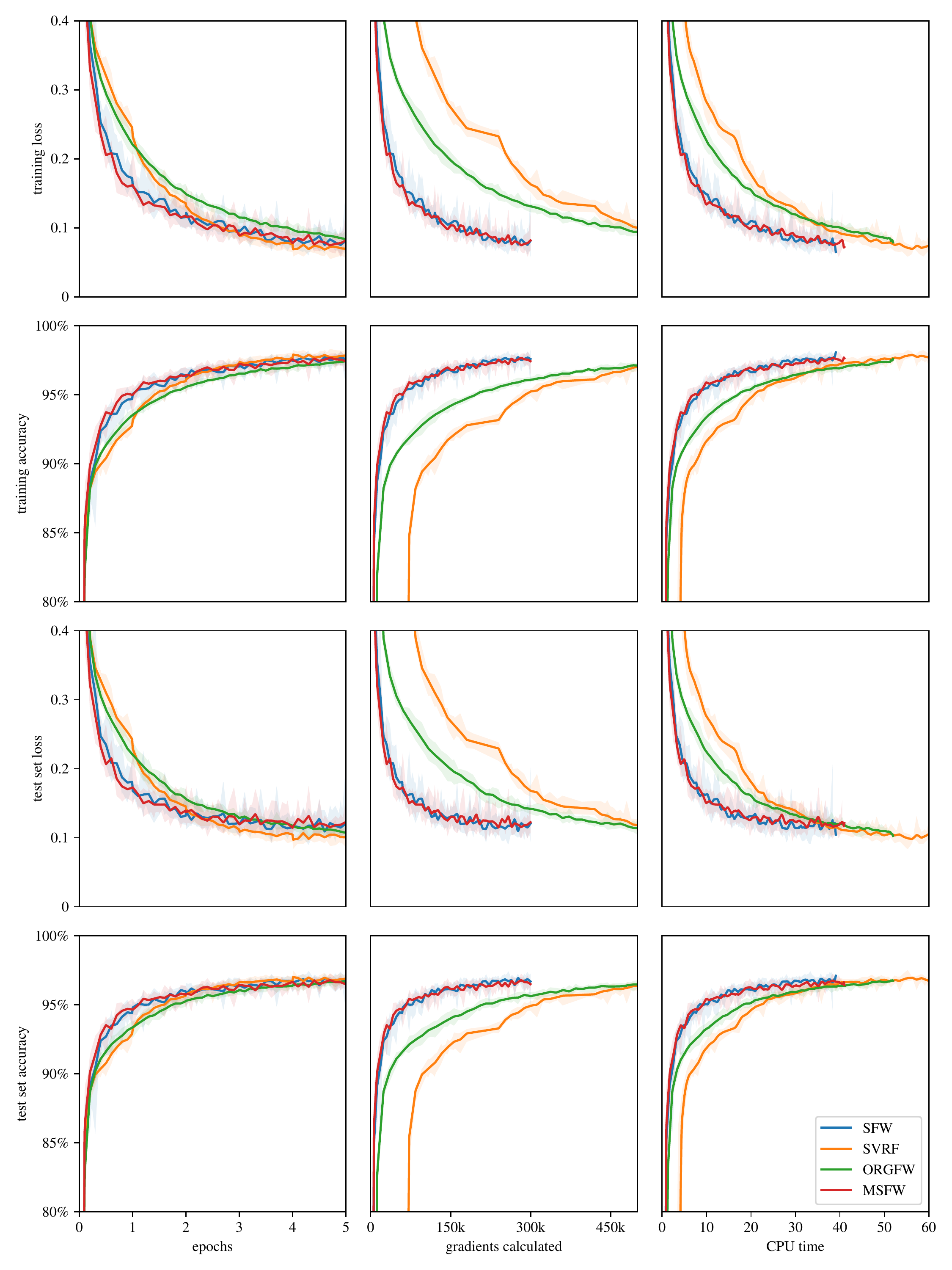}}
\caption{Comparing different stochastic Frank--Wolfe algorithms to train a fully connected Neural Network on the MNIST dataset. The fully connected Neural Network consists of two hidden layers of size 64 with ReLU activations. The $L^2$-constraints were determined according to Equation~\eqref{eq:initialization_diameter} with $w = 300$. For SFW batch size was 50 and learning rate was $0.3$ set according to Equation~\eqref{eq:decouple_diameter}. For MSFW batch size was 50, learning rate was $0.1$ set according to Equation~\eqref{eq:decouple_diameter}, and momentum was $0.9$, that is $\rho = 0.1$. For SVRF batch size was 50 and learning rate was $0.1$ set according to Equation~\eqref{eq:decouple_diameter}, and the reference point was updated at the beginning of each epoch. For ORGFW batch size was 50, learning rate was $0.03$ set according to Equation~\eqref{eq:decouple_diameter}, and momentum was $0.9$, that is $\rho = 0.1$. Results are averaged over 10 runs each.
}
\label{fig:fw_comparison_vr-mnist_extended}
\end{figure}

\begin{figure}[h]
\centerline{\includegraphics[trim=8 8 8 8, clip, width=0.8\columnwidth]{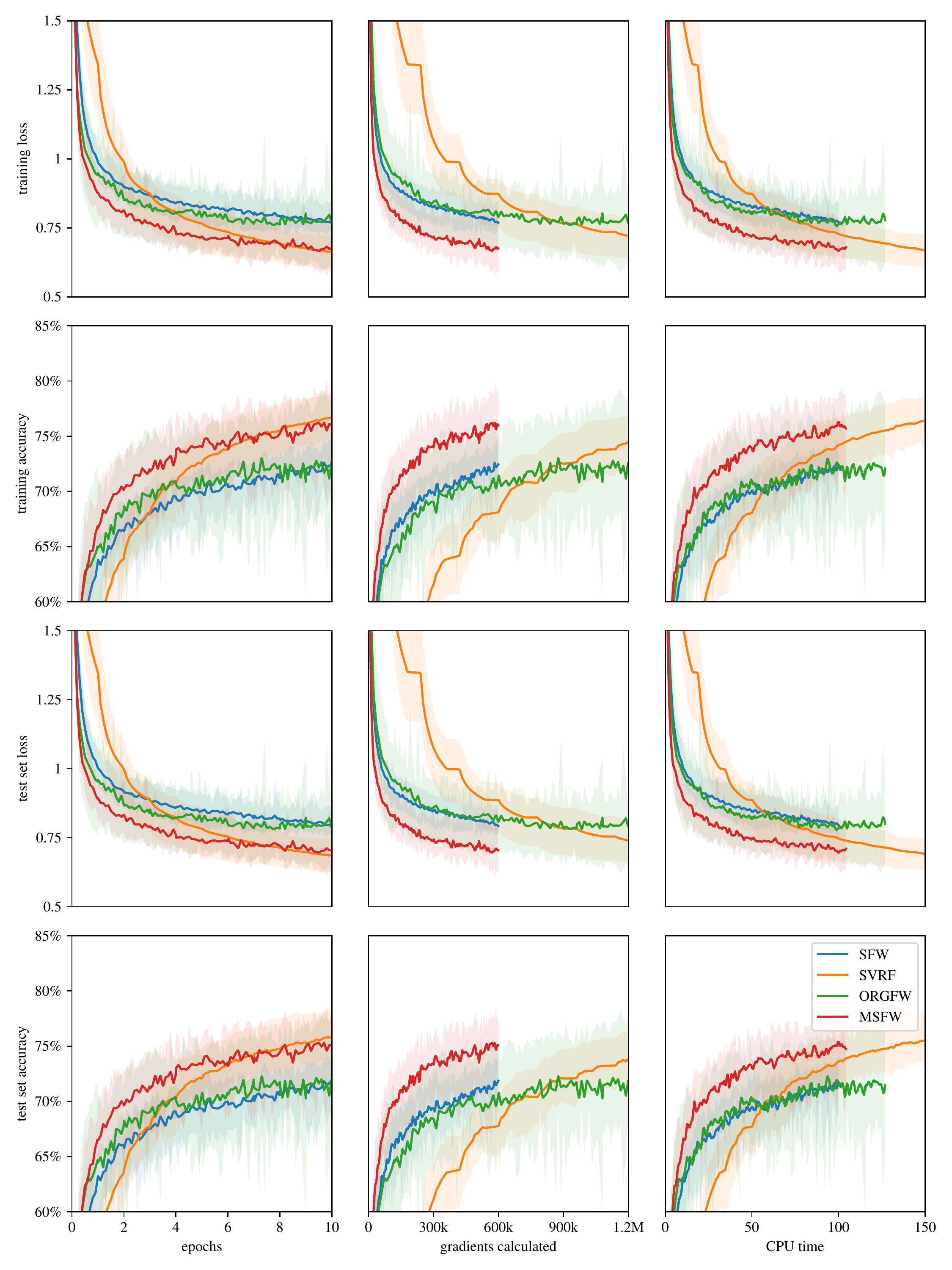}}
\caption{Comparing different stochastic Frank--Wolfe algorithms to train a fully connected Neural Network on the Fashion-MNIST dataset. The fully connected Neural Network consists of two hidden layers of size 64 with ReLU activations. The $L^1$-constraints were determined according to Equation~\eqref{eq:initialization_diameter} with $w = 1000$. For SFW batch size was 50 and learning rate was $0.1$ set according to Equation~\eqref{eq:decouple_diameter}. For MSFW batch size was 50, learning rate was $0.1$ set according to Equation~\eqref{eq:decouple_diameter}, and momentum was $0.9$, that is $\rho = 0.1$. For SVRF batch size was 50 and learning rate was $0.03$ set according to Equation~\eqref{eq:decouple_diameter}, and the reference point was updated at the beginning of each epoch. For ORGFW batch size was 50, learning rate was $0.1$ set according to Equation~\eqref{eq:decouple_diameter}, and momentum was $0.9$, that is $\rho = 0.1$. Results are averaged over 10 runs each.
}
\label{fig:fw_comparison_vr-fashion_mnist_extended}
\end{figure}

\begin{figure}[h]
\centerline{\includegraphics[trim=8 8 8 8, clip, width=0.8\columnwidth]{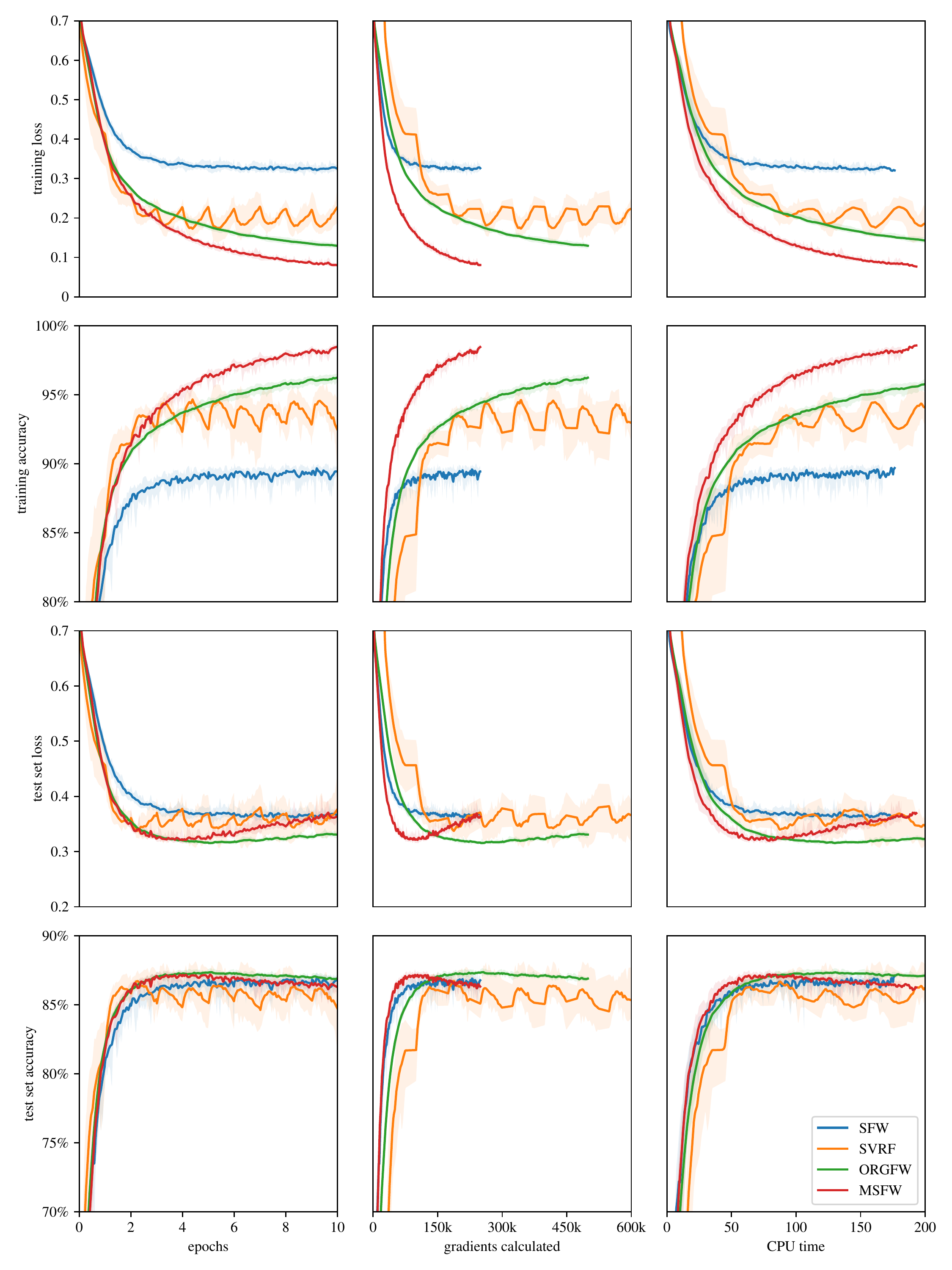}}
\caption{Comparing different stochastic Frank--Wolfe algorithms to train a fully connected Neural Network on sparse feature vectors generated from the IMDB dataset of movie reviews. The fully connected Neural Network consists of a single hidden layers of size 32 with ReLU activations. The $L^\infty$-constraints were determined according to Equation~\eqref{eq:initialization_diameter} with $w = 3$. For SFW batch size was 100 and learning rate was $0.1$ set according to Equation~\eqref{eq:decouple_diameter}. For MSFW batch size was 100, learning rate was $0.03$ set according to Equation~\eqref{eq:decouple_diameter}, and momentum was $0.9$, that is $\rho = 0.1$. For SVRF batch size was 100 and learning rate was $0.1$ set according to Equation~\eqref{eq:decouple_diameter}, and the reference point was updated at the beginning of each epoch. For ORGFW batch size was 100, learning rate was $0.03$ set according to Equation~\eqref{eq:decouple_diameter}, and momentum was $0.9$, that is $\rho = 0.1$. Results are averaged over 10 runs each.}
\label{fig:fw_comparison_vr-imdb_extended}
\end{figure}

The setup in Figure~\ref{fig:sparseness_sparse-mnist} was as follows: the fully connected Neural Network consists of two hidden layers of size 32 with ReLU activations. The convolutional Neural Network consists of a $3 \times 3$ convolutional layer with 32 channels and ReLU activations, followed by a $2 \times 2$ max pooling layer, followed by a $3 \times 3$ convolutional layer with 64 channels and ReLU activations, followed by a $2 \times 2$ max pooling layer, followed by a $3 \times 3$ convolutional layer with 64 channels and ReLU activations, followed by a fully connected layer of size 64 with ReLU activations. Both networks were always trained for 10 epochs with a batch size of 64. All $L^p$-constraints were determined according to Equation~\eqref{eq:initialization_diameter} with $w = 100$. For the $K$-sparse polytope, a radius of $300$ was used for each layer and $K$ was determined as the larger value of either $100$ or $0.1$ times the number of weights in the case of the fully connected network and the larger value of either $100$ or $0.03$ times the number of weights in the case of the convolutional network. Constrained networks were trained using SFW with momentum of $0.9$, that is $\rho = 0.1$, and a learning rate of $0.3$ using the modification suggested in Equation~\eqref{eq:gradient_rescale}. Unconstrained networks were trained using SGD with momentum of $0.9$ and a learning rate of $0.3$. When using weight decay, the parameter was set to $0.0001$. Results are averaged over 5 runs each.

\medskip

\begin{figure}[ht]
\centerline{\includegraphics[trim=8 8 8 8, clip, width=0.95\columnwidth]{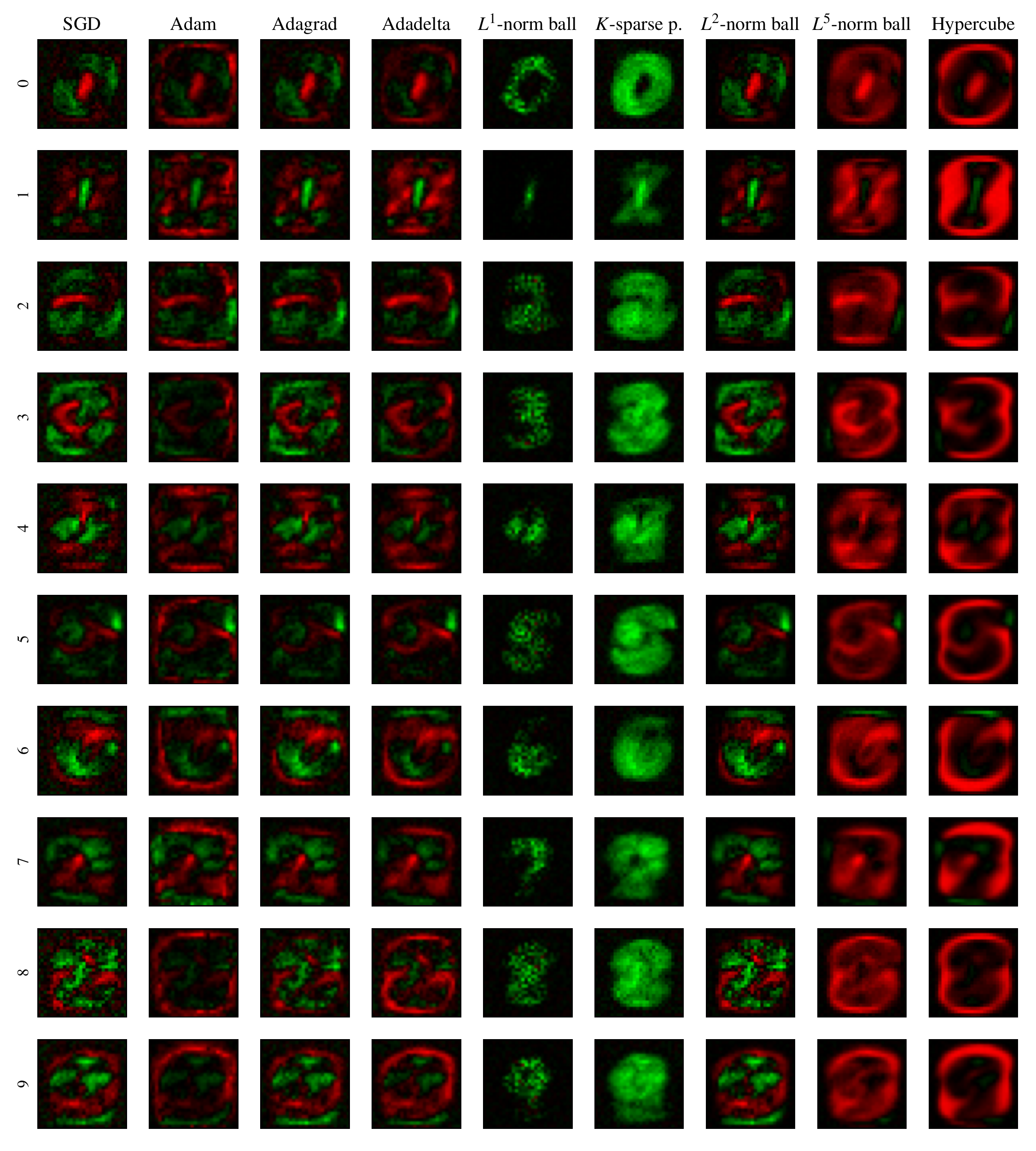}}
\caption{Visualization of the weights in a fully connected no-hidden-layer classifier trained on the MNIST dataset for 10 epochs with a batch size of 64. Red corresponds to negative and green to positive weights. The parameters of the network in the first four columns, corresponding to the SGD, Adam, Adagrad and Adadelta optimizers, were unconstrained. The networks in the remaining columns were constrained as indicated and trained using SFW with the modification in Equation~\eqref{eq:decouple_diameter}. All results are based on single runs with the corresponding settings. The hyperparameters for the constraints were $\tau = 10000$ for the $L^1$-norm ball, $\tau = 1000$ and $K = 1000$ for the $K$-sparse polytope, $\tau = 1000$ for the $L^2$-norm ball,  $\tau = 100$ for the $L^5$-norm ball and  $\tau = 1$ for the $L^\infty$-norm ball. The test set accuracies achieved were, from left to right, 92.15\%, 92.65\%, 92.42\%, 92.55\%, 86.96\%, 91.92\%, 92.17\%, 92.09\% and 90.07\%. Learning rates were from left to right, $0.1$, $0.001$, $0.1$, $1$, $0.3$, $0.3$, $0.3$, $0.1$ and $0.03$. Momentum was not used for any of the runs.}
\label{fig:mnist_visualization_complete}
\end{figure}

\begin{figure}[ht]
\centerline{\includegraphics[trim=8 8 8 8, clip, width=0.95\columnwidth]{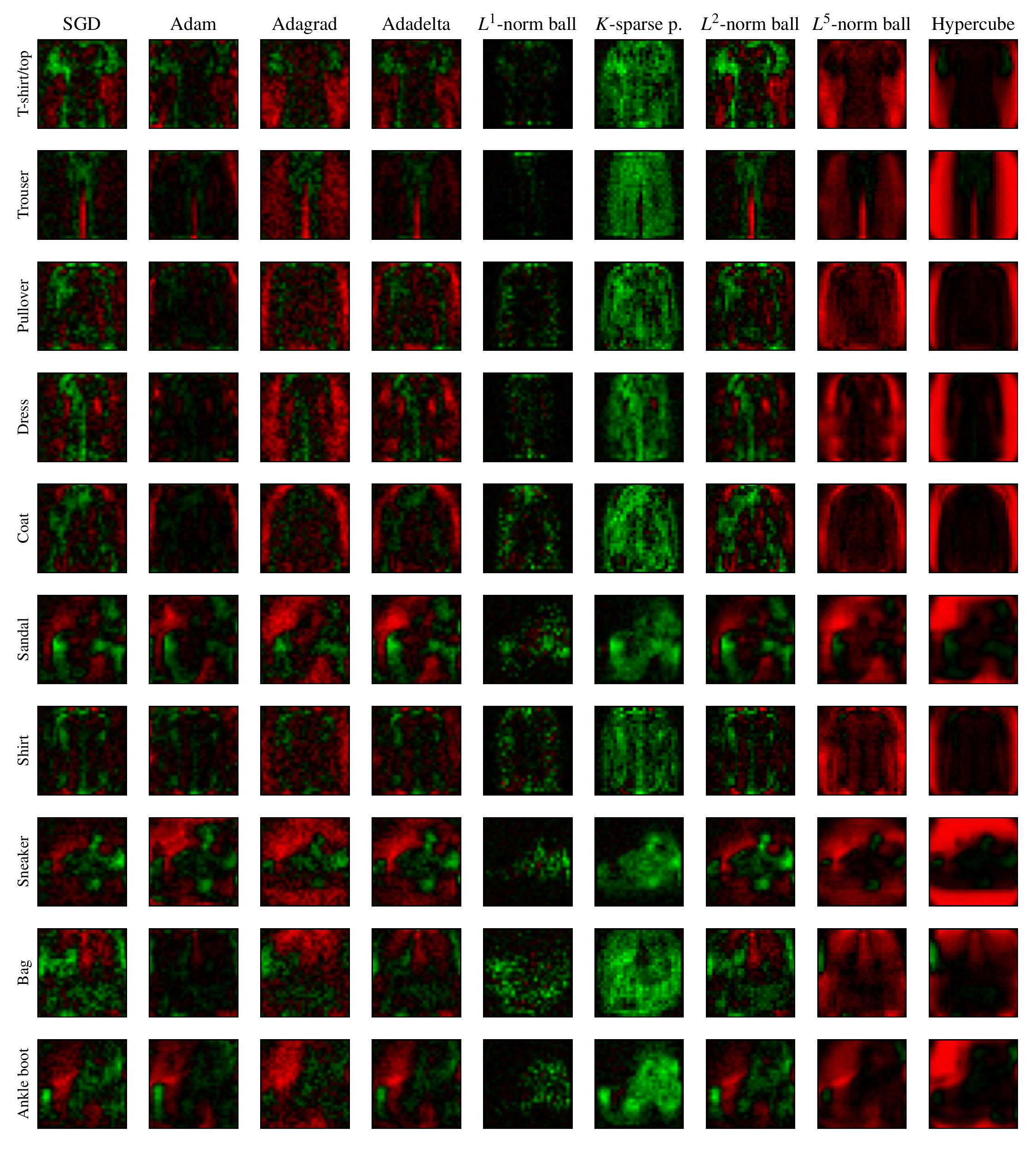}}
\caption{Visualization of the weights in a fully connected no-hidden-layer classifier trained on the Fashion-MNIST dataset for 10 epochs with a batch size of 64. Red corresponds to negative and green to positive weights. The parameters of the network in the first four columns, corresponding to the SGD, Adam, Adagrad and Adadelta optimizers, were unconstrained. The networks in the remaining columns were constrained as indicated and trained using SFW with the modification in Equation~\eqref{eq:decouple_diameter}. All results are based on single runs with the corresponding settings. The hyperparameters for the constraints were $\tau = 10000$ for the $L^1$-norm ball, $\tau = 3000$ and $K = 1000$ for the $K$-sparse polytope, $\tau = 1000$ for the $L^2$-norm ball,  $\tau = 100$ for the $L^5$-norm ball and  $\tau = 1$ for the $L^\infty$-norm ball. The test set accuracies achieved were, from left to right, 82.96\%, 84.22\%, 83.63\%, 83.46\%, 77.62\%, 82.17\%, 83.43\%, 81.97\% and 80.55\%. Learning rates were from left to right, $0.03$, $0.001$, $0.03$, $0.3$, $0.3$, $0.3$, $0.1$, $0.1$ and $0.03$. Momentum was not used for any of the runs.}
\label{fig:fashion_mnist_visualization_complete}
\end{figure}

\begin{table*}[ht]
\begin{center}
\begin{tabular}{lcccc}
\toprule
& \multicolumn{2}{c}{\bf training} & \multicolumn{2}{c}{\bf test} \\
& loss & accuracy & loss & accuracy \\
\midrule
\bf SGD without weight decay & \bf 0.0019 \tiny{\textpm 0.0004} & 99.94\% \tiny{\textpm 0.01} & 0.4916 \tiny{\textpm 0.0177} & 93.32\% \tiny{\textpm 0.17} \medskip \\ 
\bf SGD with weight decay & 0.0020 \tiny{\textpm 0.0002} & \bf 99.97\% \tiny{\textpm 0.01} & 0.2848 \tiny{\textpm 0.0066} & 94.20\% \tiny{\textpm 0.19} \medskip \\ 
\bf SFW with $L^2$-constraints & 0.0034 \tiny{\textpm 0.0003} & 99.96\% \tiny{\textpm 0.01} & \bf 0.2425 \tiny{\textpm 0.0074} & \bf 94.46\% \tiny{\textpm 0.13} \medskip \\ 
\bf SFW with $L^\infty$-constraints & 0.0028 \tiny{\textpm 0.0002} & 99.95\% \tiny{\textpm 0.01} & 0.2744 \tiny{\textpm 0.0104} & 94.20\% \tiny{\textpm 0.19} \\ 
\bottomrule
\end{tabular}
\end{center}
\caption{DenseNet121 trained for 200 epochs on the CIFAR-10 dataset with batch size 64. SGD was used with Nesterov momentum of $0.9$. When using weight decay, the parameter was set to $0.0001$. SFW was used with momentum of $0.9$, that is $\rho = 0.1$, and the modification suggested in Equation~\eqref{eq:gradient_rescale}. $L^2$-constraints were determined according to Equation~\eqref{eq:initialization_diameter} with $w = 13$. $L^\infty$-constraints were determined according to Equation~\eqref{eq:initialization_diameter} with $w = 30$. The learning rate was initialized to $\alpha = 0.1$ and decreased by a factor of $0.1$ every 90 epochs. Additionally, it was automatically decreased by a factor of $0.7$ if the 5-epoch average of the loss was strictly greater than the 10-epoch average and increased by a factor of by roughly $1.06$ if the opposite was true. All results are based on at least three runs with the corresponding settings.}
\end{table*}

\begin{table*}[ht]
\begin{center}
\begin{tabular}{lcccc}
\toprule
& \multicolumn{2}{c}{\bf training} & \multicolumn{2}{c}{\bf test} \\
& loss & accuracy & loss & accuracy \\
\midrule
\bf SGD without weight decay & 0.0007 \tiny{\textpm 0.0001} & 99.98\% \tiny{\textpm 0.01} & 0.4446 \tiny{\textpm 0.0107} & 94.44\% \tiny{\textpm 0.12}  \medskip \\ 
\bf SGD with weight decay & 0.0008 \tiny{\textpm 0.0001} & \bf 100.00\% \tiny{\textpm 0.00} & \bf 0.2226 \tiny{\textpm 0.0127} & \bf 95.13\% \tiny{\textpm 0.11} \medskip \\ 
\bf SFW with $L^2$-constraints & 0.0010 \tiny{\textpm 0.0001} & 99.99\% \tiny{\textpm 0.01} & 0.2843 \tiny{\textpm 0.0101} & 94.58\% \tiny{\textpm 0.18} \medskip \\ 
\bf SFW with $L^\infty$-constraints & \bf 0.0006 \tiny{\textpm 0.0001} & 99.99\% \tiny{\textpm 0.00} & 0.5246 \tiny{\textpm 0.0798} & 94.03\% \tiny{\textpm 0.35} \\ 
\bottomrule
\end{tabular}
\end{center}
\caption{WideResNet28x10 trained for 200 epochs on the CIFAR-10 dataset with batch size 128. SGD was used with Nesterov momentum of $0.9$. When using weight decay, the parameter was set to $0.0001$. SFW was used with momentum of $0.9$, that is $\rho = 0.1$, and the modification suggested in Equation~\eqref{eq:gradient_rescale}. $L^2$-constraints were determined according to Equation~\eqref{eq:initialization_diameter} with $w = 30$. $L^\infty$-constraints were determined according to Equation~\eqref{eq:initialization_diameter} with $w = 100$. The learning rate was initialized to $\alpha = 0.1$ and decreased by a factor of $0.2$ at epoch 60, 120 and 160. Additionally, it was automatically decreased by a factor of $0.7$ if the 5-epoch average of the loss was strictly greater than the 10-epoch average and increased by a factor of by roughly $1.06$ if the opposite was true. All results are based on five runs with the corresponding settings.}
\end{table*}

\begin{table*}[ht]
\begin{center}
\begin{tabular}{lcccc}
\toprule
& \multicolumn{2}{c}{\bf training} & \multicolumn{2}{c}{\bf test} \\
& loss & accuracy & loss & accuracy \\
\midrule
\bf SGD without weight decay & \bf 0.0094 \tiny{\textpm 0.0012} & 99.71\% \tiny{\textpm 0.04} & 1.9856 \tiny{\textpm 0.0278} & 76.82\% \tiny{\textpm 0.25} \medskip \\ 
\bf SGD with weight decay &  0.0113 \tiny{\textpm 0.0007} & 99.85\% \tiny{\textpm 0.02} & 1.0385 \tiny{\textpm 0.0126} & 77.50\% \tiny{\textpm 0.13}  \medskip \\ 
\bf SFW with $L^2$-constraints & 0.0152 \tiny{\textpm 0.0011} & \bf 99.94\% \tiny{\textpm 0.01} & \bf 0.8188 \tiny{\textpm 0.0039} & \bf 78.88\% \tiny{\textpm 0.10}  \medskip \\ 
\bf SFW with $L^\infty$-constraints & 0.0131 \tiny{\textpm 0.0015} & 99.82\% \tiny{\textpm 0.02} & 1.0462 \tiny{\textpm 0.0140} & 76.54\% \tiny{\textpm 0.50} \\ 
\bottomrule
\end{tabular}
\end{center}
\caption{GoogLeNet trained for 200 epochs on the CIFAR-100 dataset with batch size 128. SGD was used with Nesterov momentum of $0.9$. When using weight decay, the parameter was set to $0.0001$. SFW was used with momentum of $0.9$, that is $\rho = 0.1$, and the modification suggested in Equation~\eqref{eq:gradient_rescale}. $L^2$-constraints were determined according to Equation~\eqref{eq:initialization_diameter} with $w = 10$. $L^\infty$-constraints were determined according to Equation~\eqref{eq:initialization_diameter} with $w = 30$. The learning rate was initialized to $\alpha = 0.1$ and decreased by a factor of $0.2$ at epoch 60, 120 and 160. Additionally, it was automatically decreased by a factor of $0.7$ if the 5-epoch average of the loss was strictly greater than the 10-epoch average and increased by a factor of by roughly $1.06$ if the opposite was true. All results are based on three runs with the corresponding settings.}
\end{table*}

\begin{table*}[ht]
\begin{center}
\begin{tabular}{lcccccc}
\toprule
& \multicolumn{3}{c}{\bf training} & \multicolumn{3}{c}{\bf test} \\
& loss & accuracy & top-5 accuracy & loss & accuracy & top-5 accuracy \\
\midrule
\bf SGD without weight decay & \bf 0.9409 & \bf 76.95\% & \bf 91.45\% & 1.2945 & 71.06\% & 89.77\%  \medskip \\ 
\bf SGD with weight decay & 1.0107 & 75.57\% & 91.10\% & \bf 0.9957 & \bf 74.89\% & \bf 92.30\%  \medskip \\ 
\bf SFW with $L^2$-constraints & 1.1931 & 71.85\% & 89.04\% & 1.0542 & 73.46\% & 91.47\% \medskip \\ 
\bf SFW with $L^\infty$-constraints & 1.1871 & 71.64\% & 88.89\% & 1.1111 & 72.22\% & 90.72\% \\ 
\bottomrule
\end{tabular}
\end{center}
\caption{
	Results of a DenseNet121 network trained for 90 epochs on the ImageNet dataset with batch size 256. SGD was used with Nesterov momentum of $0.9$. When using weight decay, the parameter was set to $0.0001$. SFW was used with momentum of $0.9$, that is $\rho = 0.1$, and the modification suggested in Equation~\eqref{eq:gradient_rescale}. $L^2$-constraints were determined according to Equation~\eqref{eq:initialization_diameter} with $w = 30$. $L^\infty$-constraints were determined according to Equation~\eqref{eq:initialization_diameter} with $w = 100$. The learning rate was initialized to $\alpha = 0.1$ and decreased by a factor of $0.1$ every 30 epochs. Additionally, it was automatically decreased by a factor of $0.7$ if the 5-epoch average of the loss was strictly greater than the 10-epoch average and increased by a factor of by roughly $1.06$ if the opposite was true.
} 
\label{table:imagenet-densenet}
\end{table*}

\begin{table*}[ht]
\begin{center}
\begin{tabular}{lcccccc}
\toprule
& \multicolumn{3}{c}{\bf training} & \multicolumn{3}{c}{\bf test} \\
& loss & accuracy & top-5 accuracy & loss & accuracy & top-5 accuracy \\
\midrule
\bf SGD without weight decay & \bf 0.5382 & \bf 86.77\% & \bf 95.31\% & 1.6467 & 70.15\% & 88.79\%  \medskip \\ 
\bf SGD with weight decay & 0.6889 & 83.16\% & 94.37\% & 0.9810 & \bf 76.09\% & \bf 92.98\%  \medskip \\ 
\bf SFW with $L^2$-constraints &  0.9257 & 78.12\% & 92.18\% & \bf 0.9598 & 75.77\% & 92.64\%  \medskip \\ 
\bf SFW with $L^\infty$-constraints & 0.9476 & 76.92\% & 91.69\% & 1.057 & 73.95\% & 91.44\% \\ 
\bottomrule
\end{tabular}
\end{center}
\caption{
	Results of a ResNeXt50 network trained for 90 epochs on the ImageNet dataset with batch size 256. SGD was used with Nesterov momentum of $0.9$. When using weight decay, the parameter was set to $0.0001$. SFW was used with momentum of $0.9$, that is $\rho = 0.1$, and the modification suggested in Equation~\eqref{eq:gradient_rescale}. $L^2$-constraints were determined according to Equation~\eqref{eq:initialization_diameter} with $w = 30$. $L^\infty$-constraints were determined according to Equation~\eqref{eq:initialization_diameter} with $w = 100$. The learning rate was initialized to $\alpha = 0.1$ and decreased by a factor of $0.1$ every 30 epochs. Additionally, it was automatically decreased by a factor of $0.7$ if the 5-epoch average of the loss was strictly greater than the 10-epoch average and increased by a factor of by roughly $1.06$ if the opposite was true.
} 
\label{table:imagenet-resnext}
\end{table*}

\end{document}